\documentclass[journal]{IEEEtran}
\usepackage{booktabs}
\usepackage{amsmath,amsfonts}
\usepackage{mathrsfs}
\usepackage{algorithm}
\usepackage{algorithmic}
\usepackage{array}
\usepackage[caption=false,font=scriptsize,labelfont=sf,textfont=sf]{subfig}
\usepackage{textcomp}
\usepackage{cite} 
\usepackage{stfloats}
\usepackage{url}
\usepackage{verbatim}
\usepackage{graphicx}
\usepackage{multirow}
\usepackage[normalem]{ulem}
\useunder{\uline}{\ul}{}
\usepackage{mathrsfs}
\usepackage{color}

\usepackage{mathabx}
\IEEEoverridecommandlockouts

\usepackage{amssymb}
\usepackage{amsthm}
\usepackage{indentfirst}
\usepackage{makecell}
\usepackage{float}
\usepackage{xcolor}
\usepackage[bookmarks=false,colorlinks,citecolor=black, filecolor=black, linkcolor=black,urlcolor=black]{hyperref}
\newtheorem{theorem}{Theorem}

\newtheorem{lemma}{Lemma}
\newtheorem{definition}{Definition}

\usepackage{wasysym}
\usepackage{threeparttable}

\newcommand{\old}{\text{old}}

\def\BibTeX{{\rm B\kern-.05em{\sc i\kern-.025em b}\kern-.08em
    T\kern-.1667em\lower.7ex\hbox{E}\kern-.125emX}}
\usepackage{balance}

\begin{document}

\title{Topology-Assisted Spatio-Temporal Pattern Disentangling for Scalable MARL in Large-scale Autonomous Traffic Control} 

\author{
Rongpeng Li, Jianhang Zhu, Jiahao Huang, Zhifeng Zhao, and Honggang Zhang

\thanks{
    Rongpeng Li and Jianhang Zhu and Jiahao Huang are with the College of Information Science and Electronic Engineering, Zhejiang University, Hangzhou 310027, China (e-mail: \{lirongpeng; zhujh20; 22331083\}@zju.edu.cn).
   
   Zhifeng Zhao is with Zhejiang Lab, Hangzhou, China as well as the College of Information Science and Electronic Engineering, Zhejiang University, Hangzhou 310027, China (e-mail: zhaozf@zhejianglab.com).

   Honggang Zhang is with City University of Macau, Macau, China (email: hgzhang@cityu.edu.mo).
  }
}

\maketitle

\begin{abstract}
  Intelligent Transportation Systems (ITSs) have emerged as a promising solution towards ameliorating urban traffic congestion, with Traffic Signal Control (TSC) identified as a critical component. 
  Although Multi-Agent Reinforcement Learning (MARL) algorithms have shown potential in optimizing TSC through real-time decision-making, their scalability and effectiveness often suffer from large-scale and complex environments. 
  Typically, these limitations primarily stem from a fundamental mismatch between the exponential growth of the state space driven by the environmental heterogeneities and the limited modeling capacity of current solutions. 
  To address these issues, this paper introduces a novel MARL framework that integrates Dynamic Graph Neural Networks (DGNNs) and Topological Data Analysis (TDA), aiming to enhance the expressiveness of environmental representations and improve agent coordination.
  Furthermore, inspired by the Mixture of Experts (MoE) architecture in Large Language Models (LLMs), a topology-assisted spatial pattern disentangling (TSD)-enhanced MoE is proposed, which leverages topological signatures to decouple graph features for specialized processing, thus improving the model’s ability to characterize dynamic and heterogeneous local observations.
  The TSD module is also integrated into the policy and value networks of the Multi-agent Proximal Policy Optimization (MAPPO) algorithm, further improving decision-making efficiency and robustness.
  Extensive experiments conducted on real-world traffic scenarios, together with comprehensive theoretical analysis, validate the superior performance of the proposed framework, highlighting the model’s scalability and effectiveness in addressing the complexities of large-scale TSC tasks.

\end{abstract}

  
\section{Introduction}\label{sec1}

\IEEEPARstart{T}o effectively combat exacerbating traffic congestion issues \cite{8643543}, Intelligent Transportation Systems (ITS) emerge as a promising solution \cite{5430544, zhang2024irregular}, wherein Traffic Signal Control (TSC) stands out as a critical component to regulate the stop-and-go flow of vehicles \cite{8600382}.
However, traditional TSC algorithms, including pre-defined fixed-time strategies \cite{miller1963settings} and hand-crafted rules \cite{gartner2002arterial}, struggle to react in real time to highly dynamic traffic conditions due to their heuristic and over-simplistic assumptions. 
In contrast, recent advances in Artificial Intelligence (AI), particularly Multi-Agent Reinforcement Learning (MARL) \cite{9146378}, offer more flexible \cite{wei2018intellilight} and practical solutions \cite{varaiya2013max, wei2019presslight} by explicitly accounting for interdependencies among agents \cite{zhang2021multi}.
Meanwhile, although ongoing developments in Connected and Automated Vehicles (CAVs) have stimulated increasing interest in vehicle–infrastructure coordination for traffic management, the realization of a CAV-dominated environment is still expected to take decades \cite{10.1145/3447556.3447565}.
In light of these considerations, the optimization of conventional TSC systems remains a highly relevant and impactful research direction \cite{zhang2024survey}. 

Despite recent advances, MARL methods in TSC also encounter significant performance deficiencies in partially observable large-scale scenarios, owing to the difficulty of tackling heterogeneities therein \cite{varaiya2013max, wei2019presslight, 10144471}. 
As illustrated in Fig. \ref{fig: Heter}, such heterogeneity arises from the irregularities in traffic intersection configurations, including the diverse junction types (e.g., cross intersections, T-junctions) and the varying lane numbers (e.g., two-lane roads versus four-lane roads). 
To circumvent these disparities, prior approaches typically adopt a uniform representation for individual agents (i.e., Traffic Light Controllers, or TLCs), such as aggregating coarse-grained metrics at the intersection level \cite{9405489} or collecting per-lane features with a mask matrix for dimension alignment across junctions \cite{10144471}.
Although these methods ensure basic input compatibility across agents, they inevitably obscure critical structural distinctions and introduce noise, thereby impairing the expressive fidelity of the environment.
Moreover, along with the exponentially increasing number of agents and exploding diversity of their traffic patterns, current strategies, which are designed with fixed or shallow architectures, lack sufficient representational capacity to distinguish fine-grained local variations. 
The drawback not only hinders generalization but also exacerbates the risk of overfitting \cite{pmlr-v235-obando-ceron24b}. 
In essence, the overall performance of MARL in the TSC task is fundamentally constrained by its capacity to accommodate the complexity of large-scale scenarios.
Hence, there emerges a strong incentive to develop a scalable MARL framework, competently adapted to a large number of agents with high-dimensional state space.

To address the irregularities, a widely adopted methodology is to model each intersection as a graph and process it using Graph Neural Networks (GNNs), where vertices and edges represent road segments and their connections, respectively.
Given the dynamic nature of traffic flow, Dynamic Graph Neural Networks (DGNNs), which are capable of refining temporal feature evolution, are particularly eligible as the backbone of our framework for environment awareness.
Although DGNN models, such as Temporal Graph Network (TGN) \cite{rossi2020temporal}, can effectively capture intricate interactions within a graph, the resulting embeddings often lack sufficient expressiveness for downstream decision-making.
The limitation is largely attributed to the neglect of intrinsic topological patterns across distributed intersections, which are crucial for TSC tasks \cite{green2020guide}. 
To mitigate the endemic in conventional GNN-based methods, TOGL \cite{horn2021topological} proposes incorporating topological signatures, derived from Topological Data Analysis (TDA) \cite{munch2017user}, as part of the input features to enhance the expressiveness and improve performance on graph classification tasks.
In parallel, as compact and stable representations of graphs \cite{MR1867354}, topological patterns may also serve as holistic anchors for decomposing complex graphs into more tractable substructures, thereby enabling fine-grained representation learning.
However, the over-reliance on global topological descriptors in TOGL overlooks critical local structural variations \cite{zhang2023expressive}, necessitating a re-design to improve the expressive capacity of DGNN models in large-scale TSC tasks.

\begin{figure}[t]
  \centering
  \includegraphics[width=0.48 \textwidth]{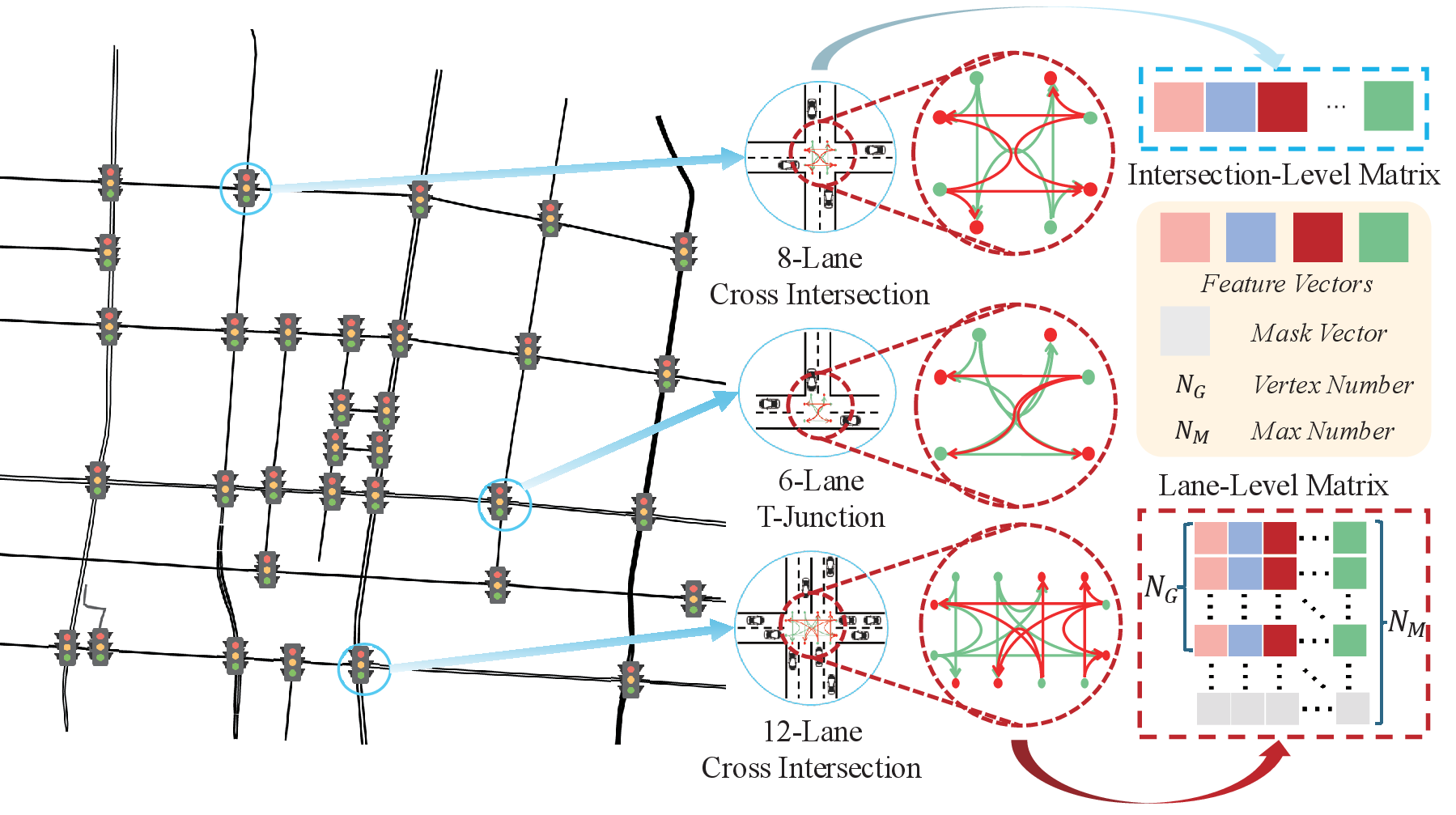}
  \caption{The illustration of intersection heterogeneities in the real-word map of Hangzhou, where the red and green lines represent possible traffic flow directions. Here, $N_G$ represents the number of vertex within each sub-graph, while $N_M$ is the maximum number of vertices across all intersections used for dimension alignment.}
  \label{fig: Heter}
\end{figure}

In this paper, we propose a Topology-assisted Spatial Disentangling (TSD)-enhanced Mixture of Experts (TMoE) architecture, tailored for scalable MARL in partially observable and large-scale environments such as TSC. 
Inspired by the demonstrated scalability of the MoE paradigm \cite{cai2025survey}, particularly the SoftMoE variant \cite{puigcerver2024from} in Large Language Models (LLMs), TMoE introduces a topology-aware routing mechanism coupled with specialized representation refinement to address the heterogeneous structural variations in traffic intersection configurations.
Specifically, by incorporating the TSD module, TMoE classifies and routes graph vertices to their pertinent experts based on local topological similarities.
This enables modular observation processing via a TGN-based backbone and facilitates more adaptive decision-making within the MARL framework.
Notably, TMoE supports information aggregation across topologically similar yet spatially distant vertices, thus alleviating the expressiveness bottleneck imposed by local smoothness in conventional GNNs \cite{rusch2022gradient}.
Through such topology-guided atomized learning, TGN with TMoE not only overcomes the limitation of TOGL to characterize local structural variations \cite{zhang2023expressive}, but also significantly enhances the model’s capacity to represent complex, large-scale state spaces. 
When applied to the decision-making module, each expert in TMoE can further specialize its policy by adapting to distinct topological patterns, promoting both robustness and interpretability.
Finally, our theoretical analysis demonstrates that TSD provides greater expressive power than TOGL, enabling more nuanced and structure-aware representations of heterogeneous traffic scenarios for downstream TSC tasks.
 In brief, the contributions of this paper are summarized as follows:
 \begin{itemize}
  \item Towards a scalable MARL solution for large-scale and partially observable TSC environments, we propose a TSD-enhanced MoE architecture (TMoE) to enable fine-grained, topology-aware representation of environmental observations. TGN-TMoE, a further integration of TMoE and TGN-based perception backbone, effectively mitigates the structural heterogeneity-induced expressivity bottleneck of conventional GNNs.
  \item We provide a theoretical framework to demonstrate the superior expressiveness of TGN-TMoE. Notably, our analysis proves that the model subsumes TOGL’s expressiveness \cite{horn2021topological}, and surpasses it in scenarios where local topological differences are critical.
  \item Extensive numerical experiments on traffic networks from three metropolitan cities validate the superiority of TGN-TMoE-empowered Multi-Agent Proximal Policy Optimization (MAPPO) over GLOSA \cite{suzuki2018new}, PressLight \cite{wei2019presslight}, and CoTV \cite{10144471}. Ablation studies further verify the necessity and effectiveness of the TSD module and its integration with TGN \cite{rossi2020temporal} for improved representation learning.
\end{itemize}

  The remainder of this paper is organized as follows. The related works are reviewed in Section \ref{sec2}. Then, we sketch system models and formulate the problem in Section \ref{sec3}. We elaborate on the details of the proposed TGN-TMoE module for observation fusion and decision-making for TSC in Section \ref{sec4}. In Section \ref{sec5}, we present the experimental results and discussions. Finally, we draw the conclusion and future directions in Section \ref{sec6}.
  
  For convenience, we also list the mainly used notations of this paper in Table \ref{tab:notations}.

  \begin{table}[t]
    \caption{Summary of Major Notations.}
    \centering
    \label{tab:notations}
    \begin{tabular}{c|m{5.5cm}}
      \toprule
      \textbf{Notation} & \textbf{Definition} \\ \midrule
      $\mathcal{N}$ & Number of agents (traffic light controllers) \\ \hline
      $\mathcal{S}$ & Global state space of the environment \\ \hline
      $\mathcal{A}$ & Joint action space of all agents \\ \hline
      $\mathcal{P}$ & State transition function \\ \hline
      $\mathcal{R}$ & Global reward function \\ \hline
      $\mathcal{O}$ & Observation space of agents \\ \hline
      $V_{\phi}$ & Value function parameterized by $\phi$ \\ \hline
      $\mathbf{s}_i^{(t)}$ & Global state of agent $i$ at time step $t$ \\ \hline
      $\mathbf{o}_i^{(t)}$ & Local state of agent $i$ at time step $t$ \\ \hline
      $\pi_{\theta}$ & Policy function parameterized by $\theta$ \\ \hline
      $a_i^{(t)}$ & Action of agent $i$ at time step $t$ \\ \hline
      $\mathcal{N}_j$ & Neighbor set of vertex $j$ \\ \hline
      $\mathcal{V}_i, \mathcal{E}_i$ & Vertex set and edge set of sub-graph $\mathcal{G}_i$ \\ \hline
      $\mathcal{G}_i := \langle\mathcal{V}_i, \mathcal{E}_i\rangle$ & Sub-graph of agent $i$ \\ \hline
      $\mathcal{G}^+_i := \langle\mathcal{V}^+_i, \mathcal{E}^+_i\rangle$ & Augmented subgraph of agent $i$ including MF vertex \\ \hline
      $A_{ij}$ & Element of adjacency matrix $\mathbf{A}$\\ \hline
      $\mathbf{e}_{jk}$ & Edge feature between vertices $j$ and $k$ \\ \hline
      $\mathbf{Mem}_j$ & Long-term memory embedding for vertex $j$ \\ \hline
      $\mathbf{v}_j, \tilde{\mathbf{v}}_j$ & Actual and predicted feature vector of vertex $j$\\ \hline
      $\mathbf{v}_\text{MF}, \mathbf{v}'_\text{MF}$ & Global and localized MF embedding \\ \hline
      $\hat{\mathbf{v}}^{(h)}_j$ & Embedding of vertex $j$ from the $h$-th GAT head \\ \hline
      $\mathbf{T}_v, \mathbf{T}_e$ & Topological signatures for vertices and edges \\ \hline
      $\mathbf{Q}, \widetilde{\mathbf{Q}}$ & Routing weight matrices in SoftMoE \\ \hline
      $\widehat{\mathbf{X}}^{mh}_p$ & Slot generated by the $p$-th router of TMoE \\ \hline
      $\mathbf{Y}_p$ & Output of the $p$-th expert in TMoE \\ \hline
      $P$ & Number of experts in MoE architecture \\ \hline
      $H$ & Number of GAT heads \\ \hline
      $T_\text{tp}$ & Waiting threshold for vehicle teleportation \\ \hline
      ${\mathbf{V}}_i,\widetilde{\mathbf{V}}_i$ & Actual and predicted feature matrix of agent $i$ \\ \hline
      $\bar{\mathbf{v}}_\text{MF}, \tilde{\mathbf{v}}_\text{MF}$ & Temporally and spatially aligned MF embeddings \\ \hline
      $\widehat{\mathbf{V}}^{(h)}_i$ & Embedding matrix of agent $i$ from the $h$-th GAT head \\ \hline
      $x(l,t)$ & Number of vehicles queued on lane $l$ at time step $t$ \\ \hline
      $T_w(l,t)$ & Cumulative waiting time for vehicles on lane $l$ at time step $t$ \\ \hline
      $\mathcal{D}^{(0)}, \mathcal{D}^{(1)}$ & The 0-dimensional and 1-dimensional persistence diagrams \\ \hline
      $d, d_0, d_1$ & Dimension number of trained vertex features, raw features and the TDA features \\ 
      \bottomrule
    \end{tabular}
\end{table}

\section{Related Work}\label{sec2}
\subsection{Control for Traffic Light Signals}
Intelligent urban traffic management has garnered increasing attention in recent years, aiming to alleviate congestion and optimize traffic flow.
Early attempts predominantly focus on heuristic and hand-crafted rules.
For instance, the Green Light Optimal Speed Advisory (GLOSA) system \cite{suzuki2018new} recommends optimal speeds to vehicles, enabling them to traverse intersections during green lights within a specified time window.
Similarly, the Sydney Coordinated Adaptive Traffic System (SCATS) \cite{sims1980sydney} dynamically selects appropriate signal timing plans from a predefined set in response to real-time traffic conditions.
While these methods have proven effective in enhancing traffic stability and efficiency in practice, they are fundamentally constrained by their reliance on deterministic models.
Such rigidity limits their ability to cope with the highly dynamic and stochastic nature of modern urban traffic, underscoring the need for more adaptive and robust control strategies.

Reinforcement Learning (RL), known for its ability to optimize strategies in dynamic and uncertain environments, has become a powerful tool to address these challenges \cite{9146378}. 
Early efforts, such as Ref. \cite{ault2019learning}, apply single-agent RL policies like Deep Q-Learning (DQN) \cite{mnih2013playing} to individual intersections. 
While effective for isolated intersections, these approaches overlook the intricate interdependencies across the network.
To overcome this limitation, MARL introduces coordination into decision-making processes, providing a more holistic solution for TSC tasks.
For example, PressLight \cite{wei2019presslight}, based on the MaxPressure principle \cite{kouvelas2014maximum}, utilizes DQN to optimize traffic efficiency using real-time data. 
MPLight \cite{chen2020toward} further demonstrates MARL’s scalability on networks with over a thousand signals, showcasing its potential in handling large-scale systems.
Beyond TSC alone, recent research has explored integrated approaches that jointly manage CAVs and TLCs \cite{10.1145/3580305.3599839}.
In particular, CoTV employs MAPPO to coordinate vehicle speeds and traffic signals, so as to enhance overall system performance \cite{10144471}. 
However, as large-scale CAV-infrastructure systems remain decades away from full deployment \cite{10.1145/3447556.3447565}, infrastructure-centered TSC remains a practical and commercially viable research focus.

\subsection{Collaboration for Multi-Agent Systems with GNN}
 Despite the progress of MARL in TSC, most approaches fall short in fully leveraging inter-agent cooperation.
 Recent studies address this by integrating explicit collaboration mechanisms, with GNNs providing a robust framework for learning in MARL \cite{wang2022meta}.
 Systems like CoLight \cite{10.1145/3357384.3357902} and IG-RL \cite{9405489} demonstrate the advantages of incorporating Graph Attention Network (GAT) \cite{velivckovic2017graph} into feature-processing pipelines. 
 Building upon this, Ref. \cite{wang2024large} proposes a multi-layer GNN that separately models inter- and intra-junction interactions, incorporating geometric and spatial features for improved performance.
 These advancements highlight the potential of exploring other GNN variants.
 For example, TGN \cite{rossi2020temporal} or other DGNN models \cite{skarding2021foundations}, which further leverage temporal information, are particularly suited to capturing the rapid evolution of traffic patterns, thereby enabling more adaptive decision-making in TSC \cite{wang2022meta}.
 Similarly, incorporating explicit topological signatures derived from TDA tools into GNN has demonstrated further improvements in representation learning \cite{horn2021topological}.
 This approach holds promise for generating more expressive states for TSC tasks.
 However, these methods often rely on fixed or shallow network architectures, which may lack the capacity to generalize across large and complex state spaces, leading to overfitting issues. 
 Effectively addressing structural heterogeneity in dynamic graphs remains a central challenge for developing scalable and robust TSC solutions in large-scale environments.

\subsection{Scalability for Large-Scale MARL with MoE}
 Another complication arises when scaling the simulation, as the dimensionality of the state space increases, making the training process progressively more difficult. 
 Relying solely on parameter-sharing optimization strategies to train a single global model may be suboptimal in more heterogeneous and large-scale environments \cite{christianos2021scaling}.
 This observation aligns with broader research trends, such as the rise of LLMs, where MoE architectures emerge as an effective solution to tackle the heterogeneity in large-scale data \cite{cai2025survey}. 
 Technically, MoE can dynamically route input data to specialized sub-networks, or ``experts,” allowing each expert to handle specific feature types and thus enabling more calibrated and efficient learning.
 In fact, early proposals typically follow a dense design, where activate all expert networks to process every input token \cite{jacobs1991adaptive}, improving accuracy at the cost of high computational overhead.
 Subsequent innovations have emphasized the routing network’s design, leading to sparse architectures that select only the top-$K$ experts for each training instance \cite{cai2025survey}.
 Various approaches, ranging from fixed rules \cite{roller2021hash} to linear programming \cite{lewis2021base} and even with RL \cite{bengio2015conditional}, have been explored. 
 To further ensure fair training opportunities for all experts, researchers have introduced random noise injection \cite{shazeer2017outrageously}, auxiliary loss functions \cite{fedus2022switch}, and even reversing the paradigm by letting experts choose tokens rather than the other way around \cite{zhou2022mixture}.
 SoftMoE \cite{puigcerver2024from} further advances this concept by discarding the traditional token notion, converting all tokens into a set of condensed slots.
 Each expert specializes in computing on a subset of these slots, thus ensuring balanced, comprehensive training participation.
 However, these studies are generally in a Euclidean space, and the application of MoE in non-Euclidean data like GNN is still in its infancy.

This atomized learning approach has expanded beyond natural language processing into multi-modal data analysis \cite{chen2024llava} and federated learning \cite{reisser2021federated}.
Moreover, these architectures demonstrate potential in MARL, particularly in multi-task training \cite{willi2024mixture}. 
The effectiveness of SoftMoE \cite{puigcerver2024from} has been validated for the decision-making processes in algorithms like MADQN and MAPPO \cite{pmlr-v235-obando-ceron24b}.
Despite its efficacy, refining and evaluating MoE-based MARL architectures in large-scale traffic environments with significant heterogeneity and complexity poses ongoing challenges.

\section{Preliminary and System Model}\label{sec3}

\subsection{System Models}\label{sec:system_model}

\subsubsection{TLC Agent} \label{sec:tlc_agent}
 To accomplish the TSC task, we formulate the problem as a Decentralized Partially Observable Markov Decision Process (Dec-POMDP).
 This formulation provides a structured framework to handle the partial observability and decentralized decision-making inherent in multi-agent traffic signal control. 
 Formally, it is defined as an 8-tuple $\langle \mathcal{N}, \mathcal{S}, \mathcal{A}, \mathcal{P}, \mathcal{R}, \mathcal{O}, \Omega, \gamma \rangle$, where $|\mathcal{N}|$ is the number of agents (traffic light controllers) in the set $\mathcal{N}$, $\mathcal{S}$ is the global state space of the environment, $\mathcal{A}$ represents the joint action space of all agents, $\mathcal{P}$ is the state transition function, $\mathcal{R}$ is the global reward function, $\mathcal{O}$ is the observation space, $\Omega$ denotes the observation function, and $\gamma$ is the discount factor.
 Within such a framework, each TLC at intersection $i$ makes decisions (e.g., switching traffic light status) based on its local observations (i.e., the current traffic light status and surrounding traffic flow conditions) $\mathbf{o}_i^{(t)}$, which are derived from the global state $\mathbf{s}^{(t)}$ using the observation function $\Omega(\mathbf{o}^{(t)}_i \mid \mathbf{s}^{(t)}) \in \mathcal{O}$.
 Subsequently, the agent $i$ decides whether to switch its current phase following the policy $\pi_i(a_i^{(t)} \mid \mathbf{o}^{(t)}_i)$, with ${\bf a}^{(t)} := \{a_1^{(t)},\cdots, a_{|\mathcal{N}|}^{(t)} \}$. Correspondingly, the environment evolves according to the function $\mathcal{P}(\mathbf{s}^{(t+1)}|\mathbf{s}^{(t)}, {\bf a}^{(t)})$.
 To encourage cooperative behavior, they share a global reward function $\mathcal{R}(\mathbf{s}^{(t+1)}, \mathbf{a}^{(t)})$, and the system’s objective is to maximize the expected cumulative reward $\mathbb{E}\left[\sum_{t=0}^{\infty} \gamma^{t} \mathcal{R}(\mathbf{s}^{(t)}, \bf{a}^{(t)})\right]$ over all possible trajectories, where $\gamma$ denotes a discounted constant. 
 Full implementation details are provided in Section \ref{4B}.

 \subsubsection{Graph Construction} 
 Recalling the limitations posed by the homogeneous assumption for agents, we resort to employing graph-based learning to refine the observations. 
 Specifically, given the time-varying impact of traffic signals on the flow, we regard the observations of agent $i$ as a sequence of intersection-related sub-graph ``snapshots" from the very beginning $t_0$ to $t$, i.e, $\mathbf{o}^{(t)}_i$ is derived from $\{\mathcal{G}_i(t_0), \cdots, \mathcal{G}_i(t)\}$.
 Notably, $\mathcal{G}_i(t) := \langle \mathcal{V}_i(t), \mathcal{E}_i(t)\rangle$ denotes the sub-graph at $t$ for agent $i$, where each vertex in such a sub-graph corresponds to a specific lane at the current intersection, while the directional traffic flow between lanes defines the edges among vertices. 
 In other words, each vertex feature $\mathbf{v}^{i}_j(t)$ encapsulates the traffic flow information of each lane, and $\mathcal{V}_i(t) := \{\mathbf{v}^i_j(t) \mid \forall j \in \mathcal{N}_i\}$ collects the raw vertex features from a set $\mathcal{N}_i$ of neighboring lanes. 
 These features can be organized into a matrix representation $\mathbf{V}_i(t) \in \mathbb{R}^{|\mathcal{N}_i| \times d_0}$ by stacking all the vectors, where $d_0$ is the initial feature dimension of each lane.
 Meanwhile, $\mathcal{E}_i(t) := \{\mathbf{e}^i_{jk}(t)\mid A^i_{jk}(t)=1, \forall j,k \in \mathcal{N}_i\}$ denotes the edge features, which reflect the influence of traffic signals on the flow between the connected lanes, and $A^i_{jk}(t)$ is an element in agent $i$'s adjacent matrix $\mathbf{A}_i$\footnote{For simplicity, unless explicitly emphasized, we omit the superscript $i$ and the timestamp $t$ of $\mathbf{v}^i_j(t)$, $\mathbf{e}^i_{jk}(t)$, and $A^i_{jk}$ hereafter with the understanding that all terms refer to agent $i$ at time $t$.}.
 Specifically, $A^i_{jk} = 1$ indicates the existence of edge $(j, k)$.
 Unlike most existing methods \cite{10.1145/3357384.3357902}, the methodology that vertices represent lanes and edges denote traffic flow directions allows the dynamic graph to comprehensively encode both the traffic states and the control dynamics within the intersection, facilitating modular processing of vertices with similar local topological patterns.

 Based on the aforementioned construction, a sequence of sub-graphs\footnote{To differentiate it from the graph constructed by the conventional method, we use the term \textit{sub-graph} to refer to $\mathcal{G}_i$ hereafter.} can be generated for each agent $i$. 
 To enable more flexible inter-agent coordination, we can expand the receptive field of each sub-graph by incorporating additional lanes that are directly influenced by neighboring traffic signals, thereby avoiding reliance on explicit agent-level communication.
 To further facilitate the interaction between local sub-graphs and global traffic patterns, we construct an augmented sub-graph $\mathcal{G}_i^+$, inspired by the MF mechanism introduced in \cite{hao2023gat}.
 Mathematically, we add a virtual vertex $\text{MF}$ and a set of edges as
 \begin{subequations}
 \label{eq:add_v}
     \begin{align}
         &\mathcal{N}_i^+ = \mathcal{N}_i \cup \text{MF},\\
         &\mathcal{V}_i^+ = \mathcal{V}_i \cup \mathbf{v}_\text{MF},\\
         &\mathcal{E}_i^+ = \mathcal{E}_i \cup  \{ \mathbf{e}_{jM}\mid \forall j \in \mathcal{N}_{i} \},
     \end{align}
 \end{subequations}
 where each virtual edge $\mathbf{e}_{jM}$ connects the MF vertex $\mathbf{v}_{\text{MF}}$ with the peripheral vertices in $\mathcal{N}_i$, encoding abstracted global-context information. The computation of $\mathbf{v}_\text{MF}$ is detailed in Section \ref{sec:top_softmoe}.
 
 \subsection{Preliminaries}
 \subsubsection{MAPPO}\label{MAPPO}
 As a widely adopted MARL approach, MAPPO \cite{10144471} learns decentralized policies for multiple agents under a centralized training framework, using shared or individualized observations.
 Notably, MAPPO facilitates the training of policies $\pi_{\theta_i}(\cdot \mid \mathbf{o}^{(t)}_i)$ and value functions $V_{\phi}(\mathbf{s}^{(t)})$ for all agents by periodically updating the target policy and value functions using samples generated by an earlier version of the policy $\pi_{\theta_{\old, i}}$ and an old evaluation function $V_{\phi_{\old}}(\mathbf{s}^{(t)})$. 
 In other words, 
 \begin{subequations}
  \begin{align}
  \label{eq:MAPPO_OBJ}
  J^{(t)}_{\pi_i}(\theta_i) &= \min \big( r(\theta_i) \hat{A}, \text{clip}\left(r(\theta_i), 1 - \epsilon, 1 + \epsilon\right) \hat{A} \big),\\
  J^{(t)}_{V}(\phi) &= \big(V_{\phi}(\mathbf{s}^{(t)}) - (\hat{A}+V_{\phi_{\old}}(\mathbf{s}^{(t)}))\big)^2,
  \end{align}
 \end{subequations}
 where a clipping function $\text{clip}(\cdot)$ is adopted to stabilize the policy update. 
 $\epsilon$ is the clip parameter, $ r(\theta_i) := \frac{\pi_{\theta_i}(\cdot \mid \mathbf{o}^{(t)}_i)}{\pi_{\theta_{\old, i}}(\cdot \mid \mathbf{o}^{(t)}_i)} $ denotes the probability ratio, measuring the likelihood of taking certain actions under the new policy versus the old one. 
 The term $\hat{A} := \sum_{l=0}^{\infty} (\gamma \lambda)^l \delta^{(t+l)}$ is the generalized advantage estimation (GAE) \cite{schulman2015high}, with the temporal-difference (TD) error $\delta^{(t)} := \mathcal{R}^{(t)} + \gamma V_{\phi_{\old}}(\mathbf{s}^{(t+1)}) - V_{\phi_{\old}}(\mathbf{s}^{(t)})$ and $\gamma$ being the discount factor. 
 To further encourage exploration, an entropy term is usually employed, which can be expressed as
\begin{equation}
  H(\pi_{\theta_i}) := -\sum\limits_{a \in \mathcal{A}} \pi_{\theta_i}(a \mid \mathbf{o}^{(t)}_i) \log \pi_{\theta_i}(a \mid \mathbf{o}^{(t)}_i).
\end{equation}

In summary, the final optimization objective of MAPPO can be formulated as
 \begin{equation}
  J_{\text{MAPPO}} = \mathbb{E}_{i}\big[ J^{(t)}_{\pi_i}(\theta_i) + J^{(t)}_{V}(\phi) + \iota H(\pi_{\theta_i})\big], \label{eq:opt_obj_final}
\end{equation}
where $\iota $ is the weight of the entropy term. 
To improve training efficiency for MARL, all agents share the relevant parameters, implying that $\pi_i = \pi_{\theta}, \forall i \in \mathcal{N}$.
 Consequently, the objective of policy training can be expressed as 
 \begin{equation}
  \pi_{\theta} = \arg \max_{\pi} \min_{\phi} J_{\text{MAPPO}}. 
 \end{equation}
 The training procedure for this fundamental framework is summarized in Algorithm \ref{alg:MAPPO_training}.

 \begin{algorithm}[tbp]
  \caption{The Training of MAPPO for TLC Agent.}
  \label{alg:MAPPO_training}
  \begin{algorithmic}[1]
  \REQUIRE \ \\
  Obtain the set of traffic light agents to control;\\
  Initialize the policy network $\Theta$ and value function $\Phi$ for all traffic light controllers through parameter sharing; \\ 
  Initialize the training episode number and other hyperparameters of MAPPO;
  \FOR {iteration $= \{1,\cdots, I\}$}
  \FOR {each train episode} 
  \FOR {t $ = \{1, \cdots, T\}$}
  \FOR {each TLC agent $i$}
  \STATE {Observe its own processed local state $\mathbf{o}^{(t)}_i$ from current state $\mathbf{s}^{(t)}$ and generate an action $a_i^{(t)}$;}
  \STATE{Calculate the value $V_i^{(t)}$ for current state $\mathbf{s}^{(t)}$;}
  \STATE {Execute action $a_i^{(t)}$, and observe the next state $\mathbf{s}^{(t+1)}_i$;}
  \ENDFOR
  \ENDFOR
  \STATE {Compute the advantage estimates $\hat{A}^{(t)}$ for each time step $t$;}
  \ENDFOR
  \STATE {Update $\Theta$ and $\Phi$ for each agent according to Eq. \eqref{eq:opt_obj_final} via the gradient descent and Adam optimizer.}
  \ENDFOR
  \end{algorithmic}
 \end{algorithm}
 
 \subsubsection{TDA}\label{TDA}
 Notably, we resort to TDA for an enhancement to dynamic graph-based analysis.
 Under the umbrella of TDA, we can strike at a rigorous study of ``shape" from a combination of algebraic topology and other pure mathematical methodologies.
 Simplicial complexes are the central concept in algebraic topology, while a graph $\mathcal{G}$ can be seen as a low-dimensional simplicial complex that only contains $0$-simplices (vertices) and $1$-simplices (edges) \cite{horn2021topological}. 
 In other words, we only consider the $0$-dimensional vertex-level and $1$-dimensional edge-level topological features.
 
 To extend the expressivity of the topological features, \textit{persistent homology} employs a series of \textit{filtrations} to imbue a simplicial complex with scale information, where filtrations imply performing a step-by-step examination of topological spaces based on some thresholds.
 More specifically, for a sub-graph $\mathcal{G}_i^+$ whose vertices with multi-dimension features, we construct a family of $U$ vertex \textit{filtration} functions $f_u: \mathbb{R}^{d_0} \to \mathbb{R}$ for $u = 1, \ldots, U$, where $d_0$ denotes the dimension of features, and correspondingly predefine a series of threshold $\text{th}_u^{(1)} < \cdots < \text{th}_u^{(N)}$, where we assign $N$ as the number of vertices within the graph. Notably, the operator $\Phi := [f_1, \dots, f_U]$ denotes the combination of all $U$ filtration functions and is responsible 
 for mapping the features into a compact vector space $\mathbb{R}^U$ for vertices and edges within the sub-graph.
 Through the filtration, we can obtain a family of filtrated sub-graphs $ \emptyset \subseteq \mathcal{G}_u^{(1)} \subseteq \cdots \subseteq \mathcal{G}_u^{(N)} = \mathcal{G}$, where each element $\mathcal{G}_u^{(n)} := (\mathcal{V}_u^{(n)}, \mathcal{E}_u^{(n)})$ is defined with $ \mathcal{V}_u^{(n)} := \{ j \in \mathcal{N}_i^+ \mid f_u(\mathbf{v}_j) \leq \text{th}_u^{(n)} \}$ and $\mathcal{E}_u^{(n)} := \{ \mathbf{e}^i_{jk} \mid \max \{ f_u(\mathbf{v}_j), f_u(\mathbf{v}_k) \} \leq \text{th}_u^{(n)} \}$, $\forall n \in \{1,\cdots, N\}$.
 
 After completing the filtration processes for all $U$ views of the sub-graph, we can calculate the lifespan of a topological pattern (vertex or edge), denoted by a persistence $ \rho:=[b_\rho, d_\rho)$, where $b_\rho$ and $d_\rho$ are the birth and death time, respectively. 
 Corresponding to such a process, $\Upsilon$ denotes a pertinent, pre-defined \textit{filter} operator. 
 These intervals are visualized in a Persistence Diagram (PD), $\mathcal{D} := \{(b_\rho, d_\rho) \in \mathbb{R}^2 \mid b_\rho < d_\rho \} = \{\mathcal{D}_1^{(0)},\cdots, \mathcal{D}_U^{(0)},\mathcal{D}_1^{(1)}, \cdots, \mathcal{D}_U^{(1)}\}$, as the topological summary of the filtration \cite{edelsbrunner2002topological}. 
 $\mathcal{D}_u^{(l)},\forall l\in\{0,1\}$ represents the PD result for the $u$-th filtration at the $l$-th simplex (i.e., vertices and edges). 
 However, to integrate PDs into machine learning models, it remains essential to perform some suitable \textit{embedding} operations $\Psi$, which not only preserve the topological information but also ensure compatibility with graph learning \cite{carriere2020perslay, JMLR:v20:18-358}. 
 Appendix A highlights common methods to achieve $\Psi$. 
 In line with the outlined methodologies, we employ a combination of \textit{rational hat function} \cite{JMLR:v20:18-358}, \textit{triangle point transformation}, \textit{Gaussian point transformation} and \textit{line point transformation} \cite{carriere2020perslay} to transform the PDs into topological signatures and then generate vertex-level and edge-level topological representations as
 \begin{equation}
 \label{eq:tda}
 \mathbf{T}_{v}, \mathbf{T}_e = \text{MLP}\circ\Psi\circ\Upsilon\circ\Phi(\mathcal{G}),
 \end{equation}
 where $\text{MLP}$ denotes a Multi-Layer Perceptron (MLP) operator with learnable parameters.
 Each element $\mathbf{T}_v(j) \in \mathbb{R}^{d_1}$ can be interpreted as the concatenation of the aforementioned four embedding results for the vertex $j$.

 Subsequently, these TDA-derived topological features can be seamlessly integrated into AI models, (e.g., GNNs), enriching the input data with topological insights.
 As demonstrated in Ref. \cite{horn2021topological}, the incorporation of topological features has been shown to significantly enhance the expressivity of GNNs from both theoretical and experimental perspectives. 
 This improvement primarily lies in its ability to perceive global topological features, such as PDs. 

\subsubsection{TGN}\label{TGN}
Ref. \cite{rossi2020temporal} introduces a canonical DGNN framework TGN, denoted as $\text{TGN}(\cdot): \mathcal{G}_i^+{(t)} \rightarrow \widetilde{\mathbf{V}}_i{(t)} \in \mathbb{R}^{|\mathcal{N}_i^+|\times d}$, to efficiently predict the evolving features $\widetilde{\mathbf{V}}_i{(t)}$ for vertices in dynamic graphs $\mathcal{G}_i^+{(t)}$. 
 Typically, TGN comprises two primary components, i.e., a temporal learning module (including $\text{Agg}$ and $\text{GRUCell}$) and a structural learning module $\text{GAT}$.
 For the temporal patterns, it can be refined with
 \begin{subequations}
 \label{gru}
 \begin{align}
  &\mathbf{{Mem}}_j(t) = \text{Agg}\big([\mathbf{{v}}_j(t-T), \cdots, \mathbf{{v}}_j(t)]\big),\\
  &\tilde{\mathbf{v}}_j(t+1) = \text{GRUCell}\big(\tilde{\mathbf{v}}_j(t),\mathbf{{Mem}}_j(t)\big),
 \end{align}
 \end{subequations}
 where $[\mathbf{{v}}_j(t-T), \cdots, \mathbf{{v}}_j(t)]$ represents the initial features of each vertex $j \in \mathcal{N}^{+}_i$ within the time window of size $T$ preceding the current timestamp $t$.
 $\text{Agg}(\cdot)$ is implemented as an MLP-based aggregator, where a single-step time window (i.e., $T=0$) combined with a dense layer suffices to capture short-term dynamics, as evidenced in our prior study \cite{zhu2024semantics}.
 Subsequently, the GRU module \cite{chung2014empirical} represented by $\text{GRUCell}(\cdot)$ is employed to predict the long-term pattern $\tilde{\mathbf{v}}_j(t+1)$ with its previously predicted state $\tilde{\mathbf{v}}_j(t)$ and the newly derived short-term pattern $\mathbf{{Mem}}_j(t)$. 

 To achieve a holistic representation of the spatio-temporal patterns, GNNs, particularly the Multi-Head GAT (MH-GAT) models \cite{brody2021attentive}, are eligible for subsequent spatial learning.
 By means of MH-GAT, the embedding of a vertex $j$ is iteratively updated by aggregating the features of its one-hop neighbors $\mathcal{N}^{+}_i$ as well as the associated edges in the pertinent sub-graph $\mathcal{G}^{+}_i$ by
 \begin{subequations}
 \label{eq:GAT}
 \begin{align}
 \label{eq:GAT_0}
  &\hat{\mathbf{v}}^{(h)}_j = \text{GAT}^{(h)}\big(\tilde{\mathbf{v}}_{j}(t+1), {\mathcal{G}_i^+}\big),\\
  & \tilde{\mathbf{v}}_j(t+1) \leftarrow \sigma\big(\mathbf{W}_\text{mh}[\hat{\mathbf{v}}^{(0)}_j || \cdots || \hat{\mathbf{v}}^{(h)}_j]\big), \label{eq:GAT_1}
  \end{align}
 \end{subequations}
 where $\hat{\mathbf{v}}^{(h)}_j$ denotes the hidden state of vertex $j$ generated by the $h$-th head of an MH-GAT layer $\text{GAT}^{(h)}(\cdot)$. 
 Subsequently, the final output $\tilde{\mathbf{v}}_j(t+1)$ in Eq. \eqref{eq:GAT_1} is computed via an MLP with learnable parameters $\mathbf{W}_\text{mh}$, the concatenation operator $||$ and the activation function $\sigma$.
Notably, by stacking multiple GNN layers, the final embedding for each vertex incorporates information from a larger receptive field, resulting in a more expressive representation. 
Correspondingly, the vertex embedding matrix of sub-graph $i$ (i.e., agent $i$) is defined as $ \widetilde{\mathbf{V}}_i{(t)} := \begin{bmatrix} \tilde{\mathbf{v}}_j(t+1) \end{bmatrix}_{\forall j \in \mathcal{N}_i^+} \in \mathbb{R}^{|\mathcal{N}_i^+| \times d}$.

  The training objective of TGN can be formulated with a Mean Squared Error (MSE) loss,
\begin{equation}
  J_{\text{TGN}} = \frac{1}{|\mathcal{N}_i^+|} \sum\limits_{k \in \mathcal{N}_i^+} \|\tilde{\mathbf{v}}_k(t+1) - \mathbf{v}_k{(t+1)}\|^2,
  \label{eq:loss_1}
\end{equation}
 where $\tilde{\mathbf{v}}_k(t+1)$ represents the predicted features for vertex $k$ as in Eq. \eqref{eq:GAT_1}, while
 $\mathbf{v}_k{(t+1)}$ denotes the ground-truth embedding at $t+1$.
 
\subsection{Problem Formulation}\label{PF}

 We aim to enable TLC agents to determine optimal actions through a MARL algorithm. 
 Considering the inherent structural features and dynamic patterns of each agent, we employ a TGN model, described in Section \ref{TGN}, to process the raw observations of all agents.
 The graph-level embedding for agent $i$, denoted as $\mathbf{o}^{(t)}_{i}$, is then extracted via a set of readout layers. 
 Formally, it can be expressed as
 \begin{subequations}
  \begin{align}
    &\widetilde{\mathbf{V}}_i{(t)} \leftarrow \text{TGN}\big(\mathcal{G}_i^+(t)\big),\label{eq:4a}\\
    &\mathbf{o}^{(t)}_{i} = \text{readout}\big( \widetilde{\mathbf{V}}_i(t)\big). \label{eq:4b}
  \end{align}
 \end{subequations}
 The readout module contributes to counteract the non-Euclidean nature of graph data and produces a uniform graph-level representation. 
 While the specific form of $\text{readout}(\cdot)$ in Eq. \eqref{eq:4b} is flexible, we adopt a first-order statistic, i.e., \textit{Max Pooling}, as a balance between simplicity and effectiveness.
 Prior to pooling, the input features are projected into a higher-dimensional space via an MLP layer to enhance the discriminative capacity.
 Afterward, the processed compact observations $\mathbf{o}_i \in \mathbb{R}^{d_o}$, which encapsulate partial spatio-temporal patterns of the entire junction, act as the input to the MAPPO decision module. 
 
 To jointly optimize the policy performance and the accuracy of graph representation, we define a unified objective function,
 \begin{equation}
 \label{eq:tot_obj}
  J_\text{tot} = - J_{\text{MAPPO}} + \lambda_g \cdot J_{\text{TGN}},
 \end{equation}
 where $\lambda_g$ is a trade-off coefficient controlling the contribution of the TGN loss.
 This joint optimization approach leverages the differentiability of the entire model, enabling simultaneous optimization of both policy and TGN parameters via unified backpropagation. Since the representation directly informs the decision-making, simply aggregating global features while neglecting local structure differences could undermine the benefits of MAPPO in determining policy performance \cite{zhang2023expressive}. Therefore, it remains meaningful to resort to an alternative means to effectively characterize environmental dynamics, so as to maximize Eq. \eqref{eq:tot_obj}.

\begin{figure*}[!h]
  \centering
  \includegraphics[width=0.92 \textwidth]{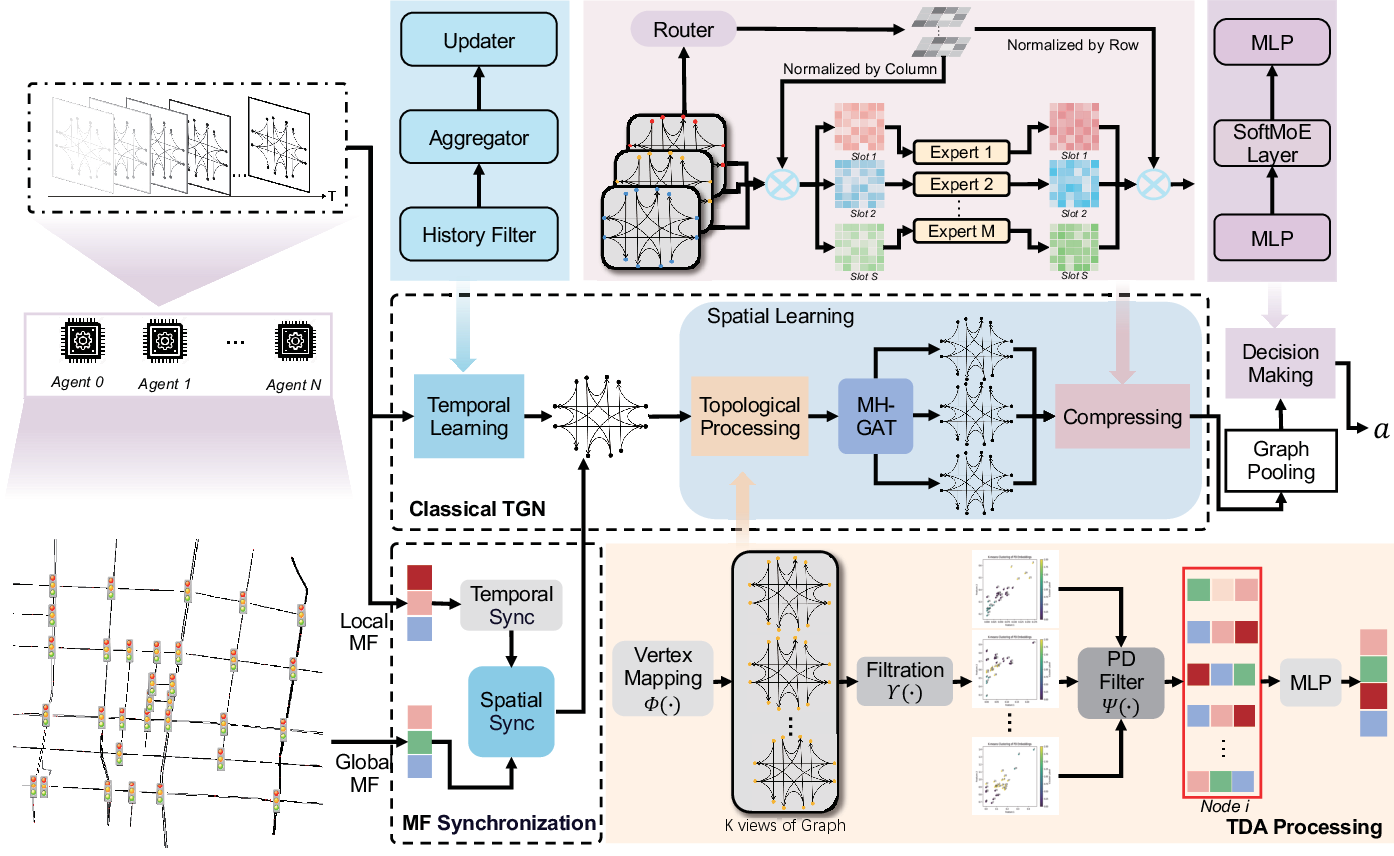}
  \caption{The pipeline of our topology-enhanced spatial disentangling model for state processing and decision-making.}
  \label{tts}
\end{figure*}
\section{Spatial-Temporal Pattern Disentangling for Observation Fusion and Decision-Making}\label{sec4}


 To yield synergistic improvement in both spatio-temporal graph modeling and decision quality, we introduce a TMoE architecture by refining the intricate dynamic graph through routing-based multiple spatially specialized experts. 
 Alongside the extraction of structural pattern with $\text{TGN}(\cdot)$, the TMoE architecture effectively orchestrates a set of specialized experts for spatial representation learning and decision making. The overall pipeline for TGN-TMoE is illustrated in Fig. \ref{tts}.
 
 Similar to SoftMoE \cite{puigcerver2024from}, TMoE also involves \textit{routing}, \textit{projecting}, and \textit{aggregation} procedures, termed as $\text{Route}(\cdot), \text{Proj}(\cdot)$ and $\text{Merge}(\cdot)$, respectively.
 Given input token matrix $\mathbf{X}$ and $P$ projection experts, the overall computation can be unified as
\begin{equation}
\label{eq:softmoe_in_1}
    \mathbf{Y} = \text{Merge}\bigg( \Big[ \text{Proj}\big( \text{Route}(\mathbf{X})_p\big)_p \Big]_{p=1}^P\bigg).
\end{equation}
 While a straightforward implementation treats each vertex $\tilde{\mathbf{v}}_j$ in the subgraph as an input token, i.e., $\mathbf{X} = \widetilde{\mathbf{V}}_i$, this strategy disregards the original topological relations among vertices and risks undermining the structural integrity of the graph.
 Moreover, due to the inherent smoothing effect in GNNs \cite{rusch2022gradient}, adjacent vertices often yield highly similar embeddings, which biases the routing mechanism toward assigning uniform scores to neighboring vertices.
 Such informational redundancy is then exacerbated through expert projections, increasing the risk of overfitting and reducing model robustness.
 These issues highlight the need for more principled routing anchors that preserve topological context and facilitate structural learning.

\subsection{Topology Enhanced Graph Representation}\label{sec:top_softmoe}
 Inspired by TOGL \cite{horn2021topological}, we introduce explicit topological priors and global coordination signals into the TGN-based learning pipeline.
 Specifically, between the sub-equations in Eq. \eqref{gru}, we introduce an additional MLP layer, $\text{MLP}_0(\cdot)$, to fuse topological features and raw features, including $\mathbf{Mem}_j$ and $\tilde{\mathbf{v}}_j$, as follows,
 \begin{equation}
 \label{eq:topo+}
  \begin{aligned}
    &\tilde{\mathbf{v}}_j \leftarrow \text{MLP}_0(\mathbf{T}_{v}(j), \tilde{\mathbf{v}}_j),\\
    &\mathbf{Mem}_j \leftarrow \text{MLP}_0(\mathbf{T}_{v}(j), \mathbf{Mem}_j), \forall j \in \mathcal{N}^+_i,\\
  \end{aligned}
\end{equation} 
 where $\mathbf{T}_{v}(j)$ is the TDA-based topological embedding generated according to Eq. \eqref{eq:tda}.
 
 Notably, to further enhance agent coordination in large-scale environments, we draw upon MF theory, 
 by adding interactions between each agent and a virtual global MF representation.
 Concretely, based on the latest available information from $t-1$, a virtual vertex, computed as $\mathbf{v}_\text{MF}(t-1) := \frac{1}{|\mathcal{N}|}\sum_{i\in\mathcal{N}}\mathbf{v}_i(t-1)$, is introduced to represent the global MF and incorporated into each sub-graph throughout the GNN's training process, as demonstrated in Eq. \eqref{eq:add_v}.
 However, such a global MF signal may become outdated and insufficiently reflect the agent-specific relevance of surrounding states.
 Therefore, we introduce a temporal and spatial synchronization mechanism, defined as 
 \begin{subequations}
   \label{eq:mf}
  \begin{align}
    & \bar{\mathbf{v}}_\text{MF}(t) = \text{MLP}_1\big(\mathbf{v}_\text{MF}(t-1)\big),\\
    & \tilde{\mathbf{v}}_\text{MF}(t) = \text{GRUCell}({\bar{\mathbf{v}}}_\text{MF}(t), \mathbf{v}'_\text{MF}),
  \end{align}
 \end{subequations}
 where the $\text{MLP}_1(\cdot)$ ensures the temporal adaptation of the global MF signal, and the local MF embedding, defined as $\mathbf{v}'_\text{MF}:= \frac{1}{|\mathcal{N}_i|}\sum_{j\in\mathcal{N}_i}\mathbf{v}_j(t)$, provides localized context for GRU-based refinement, ensuring that the global MF remains sensitive to local interactions.
 
 Finally, the enhanced subgraph $\mathcal{G}_i^+$, now enriched with both topological priors and MF elements, is processed via $\text{TGN}(\cdot)$ to extract informative temporal patterns as described in Eq. \eqref{gru}.
 Owing to the injected topological priors in Eq. \eqref{eq:topo+}, the intermediate representations from Eq. \eqref{eq:GAT_0} not only encode the raw attributes but also exhibit topology-aware characteristics, which makes them more suitable as structure-preserving tokens for MoE-based routing.
 
 \subsection{Implementations of TGN-TMoE}
 Recalling the drawbacks of regarding vertex as the input of SoftMoE, we adopt the TMoE architecture by treating the outputs of MH-GAT in Eq. \eqref{eq:GAT_0} as the ``token" while operating ``slot" calculation on the specific vertices.
 In this design, the outputs of $\text{Route}(\cdot)$ in Eq. \eqref{eq:softmoe_in_1}, denoted as $\widehat{\mathbf{X}}^\text{mh}_p$, comprise both the topological embeddings $\mathbf{T}_v$ and the original vertex features $\mathbf{V}_i$, enabling the generation of topology-sensitive routing weights for each vertex.
 Specifically, we first score the importance of individual vertices to an expert $p$, which is designed to specialize in a certain type of topological pattern, with
 \begin{equation}
     \mathbf{Q}_{jp} = \frac{\exp{\big(\mathbf{W}_p \cdot\mathbf{T}_v(j)\big)}}{\sum\limits_{j \in \mathcal{N}^+_i}\exp{\big(\mathbf{W}_p \cdot\mathbf{T}_v(j)\big)}} \cdot \mathbf{1}_{d\times H},\label{eq:moe_tt_mixing_0}
 \end{equation}
where $\mathbf{W}_p \in \mathbb{R}^{1 \times d_1}$ is the expert-oriented weight matrix to be trained while $\mathbf{1}_{d\times H}$ represents the all-ones matrix with the dimension of $d \times H$ used to broadcast the score.
Notably, $H$ denotes the number of heads.
Theoretically, the value of $\mathbf{Q}_{jp} \in \mathbb{R}^{d \times H}$ can be interpreted as quantifying the topological distance between vertex $j$ and the anchor topology $\mathbf{W}_p$.

Subsequently, these weights are then used to construct the expert-specific input slot $\widehat{\mathbf{X}}^{\text{mh}}_p$ with
 \begin{equation}
  \label{eq:moe_tt_mixing_1}
  \widehat{\mathbf{X}}^\text{mh}_p = \mathbf{Q}_p \odot \begin{bmatrix}
    \widehat{\mathbf{V}}_i^{(1)} ,
    \widehat{\mathbf{V}}_i^{(2)} ,
    \cdots ,
    \widehat{\mathbf{V}}_i^{(H)} 
    \end{bmatrix}^\intercal \in \mathbb{R}^{{|\mathcal{N}_i^+| \times d \times H}},
 \end{equation}
where $\mathbf{Q}_p$ scatters $\mathbf{Q}_{jp} \in \mathbb{R}^{d \times H}$ to align the index $j$ and $\odot$ denotes the element-wise product operator.
The matrix $\widehat{\mathbf{V}}_i^{(h)} := \begin{bmatrix} \hat{\mathbf{v}}_j^{(h)} \end{bmatrix}_{\forall j \in \mathcal{N}_i^+} \in \mathbb{R}^{|\mathcal{N}_i^+| \times d}, \forall h \in [1,2,\dots, H]$ denotes the output of the $h$-th GAT head in Eq. \eqref{eq:GAT_0} that stacks the embeddings of all vertices in $\mathcal{G}^+_i$.
As illustrated in Fig. \ref{Fig:TSD}, this routing mechanism selectively emphasizes vertices with similar topological characteristics across all GAT heads within each slot, effectively disentangling subgraphs along structural dimensions.
This topology-aware perspective makes the integration of SoftMoE into the spatial learning pipeline both theoretically grounded and practically effective.

\begin{figure}[!t]
  \centering
  \includegraphics[width=0.45 \textwidth]{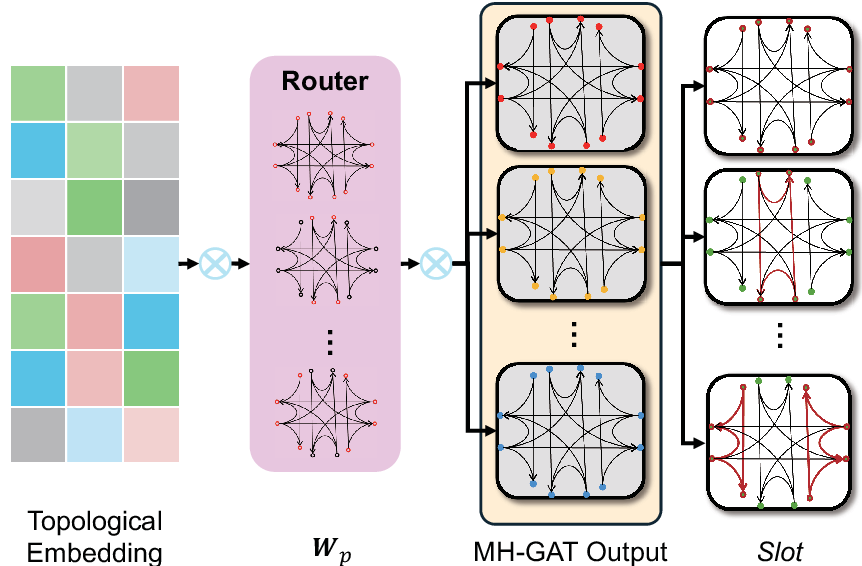}
  \caption{The illustration of the TSD module.}
  \label{Fig:TSD}
\end{figure}

 Practically, a series of MH-GAT-oriented experts $\text{Proj}(\cdot)$, functionally aligned with Eq. \eqref{eq:GAT_1}, are employed to process the disentangled representations obtained from the routing stage. 
 In detail, each expert produces
 \begin{equation}
     \mathbf{{Y}}_p = \text{MLP}_2(\widehat{\mathbf{X}}^\text{mh}_p) \in \mathbb{R}^{{|\mathcal{N}_i^+| \times d}},\label{eq:moe_tt_mixing_2_1}
 \end{equation}
 where $\mathbf{Y}_p$ denotes the output of expert $p$ applied to its corresponding slot.
 To integrate the outputs from all $P$ experts, a learnable aggregation mechanism is further adopted. 
 A combined weight matrix $\widetilde{\mathbf{Q}}$ serves as the routing coefficient for the $\text{Merge}(\cdot)$ operation in Eq. \eqref{eq:softmoe_in_1}, which can be formulated as 
\begin{subequations}
  \label{eq:moe_tt_mixing_2}
  \begin{align}
    & \widetilde{\mathbf{Q}}_{p} = \frac{\exp{\big(\sum\limits_{j \in \mathcal{N}^+_i} \mathbf{W}_{sp} \cdot \mathbf{T}_v(j) \big)_{p}}}{\sum\limits_{p = 1}^{P}\exp{\big(\sum \limits_{j \in \mathcal{N}^+_i} \mathbf{W}_{sp} \cdot \mathbf{T}_v (j)\big)_{p}}}\cdot\mathbf{1}_{|\mathcal{N}_i^+| \times d}, \label{eq:moe_tt_mixing_2_2}\\
    & \widetilde{\mathbf{V}}_i = \sum\limits^P_{p=1}\widetilde{\mathbf{Q}}_p \odot \mathbf{Y}_p,\label{eq:moe_tt_mixing_2_3}
  \end{align}
 \end{subequations} 
 where $\mathbf{W}_{sp} \in \mathbb{R}^{1 \times d_1}$ is a trainable parameter matrix governing the $p$-th expert fusion strategy.
 Finally, the aforementioned TGN-TMoE replaces the conventional implementation of Eq. \eqref{eq:GAT_1} and Eq. \eqref{eq:4a}, and the aggregated output $\widetilde{\mathbf{V}}_i$ is passed through the readout operator in Eq. \eqref{eq:4b} to yield the more fine-grained observation $\mathbf{o}^{(t)}_i$ for downstream decision-making module.

 Finally, we summarize the overall TGN-TMoE procedure in Algorithm \ref{alg:TTS}. 
 
\begin{algorithm}[t]
  \caption{Training Process of TGN-TMoE Model.}
  \label{alg:TTS}
  \begin{algorithmic}[1]
  \REQUIRE Sub-graph $\mathcal{G}^+_i$ for each junction and the mean-field embedding of the entire map $\mathcal{G}$.
  \ENSURE The representations of the whole graph-based junction.\
  \STATE{Obtain the MF results for each sub-graph with Eq. \eqref{eq:mf};}
  \STATE {Insert the customized global map embedding as a virtual vertex into $\mathcal{G}^+_i$ to obtain the extended graph $\mathcal{G}^+_i$;}
  \STATE{Obtain the PD $\mathcal{D}_i$ for $\mathcal{G}^{+}_i$;}
  \STATE{Calculate the topological signatures from the global PD $\mathcal{D}_i$ for individual vertex and edge with Eq. \eqref{eq:tda};}
  \STATE {Update each vertex feature $\mathbf{v}_j$ using memory $\mathbf{Mem}_{j}, \forall j \in \mathcal{G}^{+}_i$ via Eq. \eqref{gru};}
  \FOR{each sub-graph $\mathcal{G}^{+}_i$}
      \FOR{each layer of the GAT module}
        \STATE {Iterate the embeddings of current graph with MH-GAT module;}
        \STATE{Regard the multi-head outputs as the tokens for the MoE architecture and calculate the topology-oriented slot with Eqs. \eqref{eq:moe_tt_mixing_0} and \eqref{eq:moe_tt_mixing_1};}
        \STATE{Route slots to experts for delicate processing with Eq. \eqref{eq:moe_tt_mixing_2_1};}
        \STATE{Merge the disentangled graph-based results into compact representations via Eqs. \eqref{eq:moe_tt_mixing_2_2} and \eqref{eq:moe_tt_mixing_2_3};}
      \ENDFOR
  \ENDFOR
  \STATE{Readout the vertex-level embeddings to obtain a unified graph-level representation for decision-making.}
  \end{algorithmic}
 \end{algorithm}
 
\subsection{Theoretical Analysis of TGN-TMoE}
To theoretically validate the effectiveness of the proposed TMoE-enhanced TGN architecture (TGN-TMoE), we demonstrate its superiority over the State-Of-The-Art (SOTA) TOGL \cite{horn2021topological} in terms of the Weisfeiler-Lehman (WL) test \cite{weisfeiler1968reduction}, a widely used criterion for evaluating the discriminative capacity of graph models. 
The 1-WL test iteratively refines vertex labels by aggregating the labels of each vertex and its immediate neighbors using an injective hash function. 
This process continues until convergence or a predefined maximum number of iterations is reached.
Typically, two graphs are considered 1-WL equivalent if their final vertex labels are indistinguishable.
Building on this insight, our proposed TGN-TMoE architecture further enhances expressiveness by introducing topologically guided expert specialization.
As shown in the following theorem, TGN-TMoE exceeds the expressiveness of TOGL, thereby inheriting and advancing the theoretical guarantees associated with 1-WL test.

\begin{theorem}[Superior Expressiveness of TGN-TMoE]\label{thm:superExp}
The improved model exhibits superior expressivity power compared to TOGL in distinguishing graphs that satisfy the following conditions:
\begin{enumerate}
    \item The graphs have identical global topological features (e.g., 0-dimensional and 1-dimensional PD are equal).
    \item The graphs differ in their local topological structures (e.g., cycles or connectivity patterns).
    \item The graphs are distinguishable by the 1-WL test \cite{weisfeiler1968reduction}.
\end{enumerate}
\end{theorem}
\noindent We leave the proof in Appendix B. 

\subsection{TMoE-Based Decision Making}\label{sec4_2} 
After obtaining the environmental observations, inspired by \cite{pmlr-v235-obando-ceron24b}, we adopt a policy mixture architecture with the previously introduced TSD router to ameliorate the performance of MAPPO in large-scale TSC tasks. 
 In particular, we replace the penultimate layer of both the policy and value networks of MAPPO with the proposed TMoE architecture as well. 
 In detail, we divide each subgraph representation that has been compressed with an MLP layer into $B$ equal-size segments, effectively creating $B$ identifiable ``tokens".  
 Mathematically, the resulting policy under the TMoE-enhanced architecture can be modified as
 \begin{equation}
  \label{eq:MoE_policy}
  \pi(a^{(t)}_i \mid \mathbf{o}^{(t)}_i)= \sum\limits_{b=1}^{B} {S}^P_{b}\pi_b(a^{(t)}_i \mid \mathbf{o}^{(t)}_i),
 \end{equation}
 where $\pi_b(a^{(t)}_i \mid \mathbf{o}^{(t)}_i)$ is the expert-specific policy that is comprised by the nonlinear transformation within the TMoE framework.


\section{Experimental Results and Discussions}\label{sec5}

\subsection{Experimental Settings}\label{4B}
\subsubsection{Scenario Settings}
\begin{figure*}[tbp]
    \centering
      \subfloat[The map of Shenzhen]{\includegraphics[width = 0.33\textwidth]{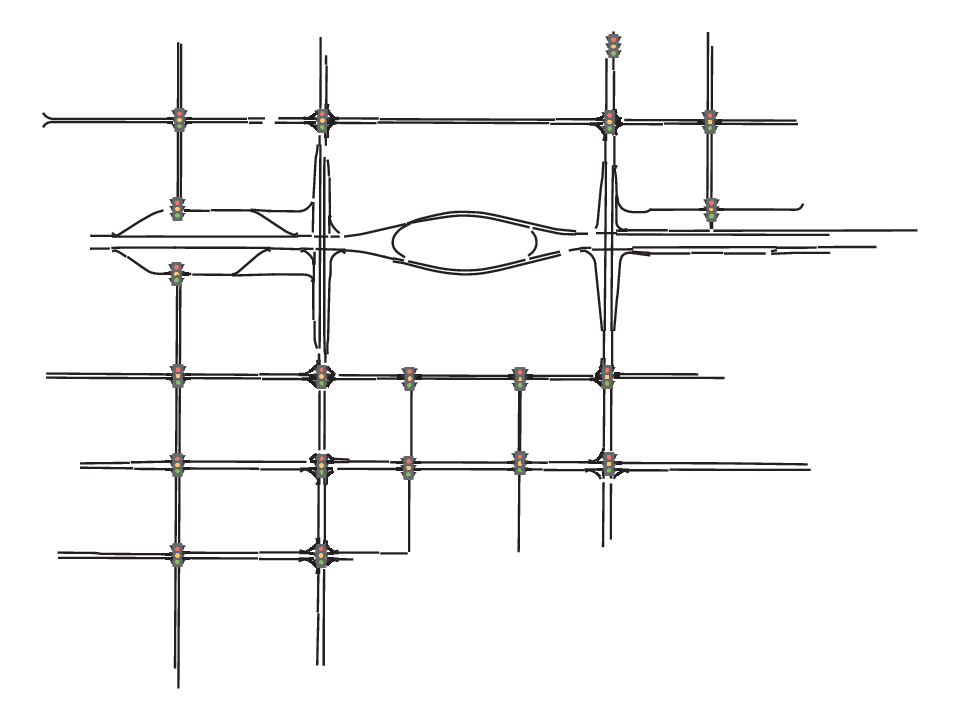}}
      \subfloat[The map of Shanghai]{\includegraphics[width = 0.33\textwidth]{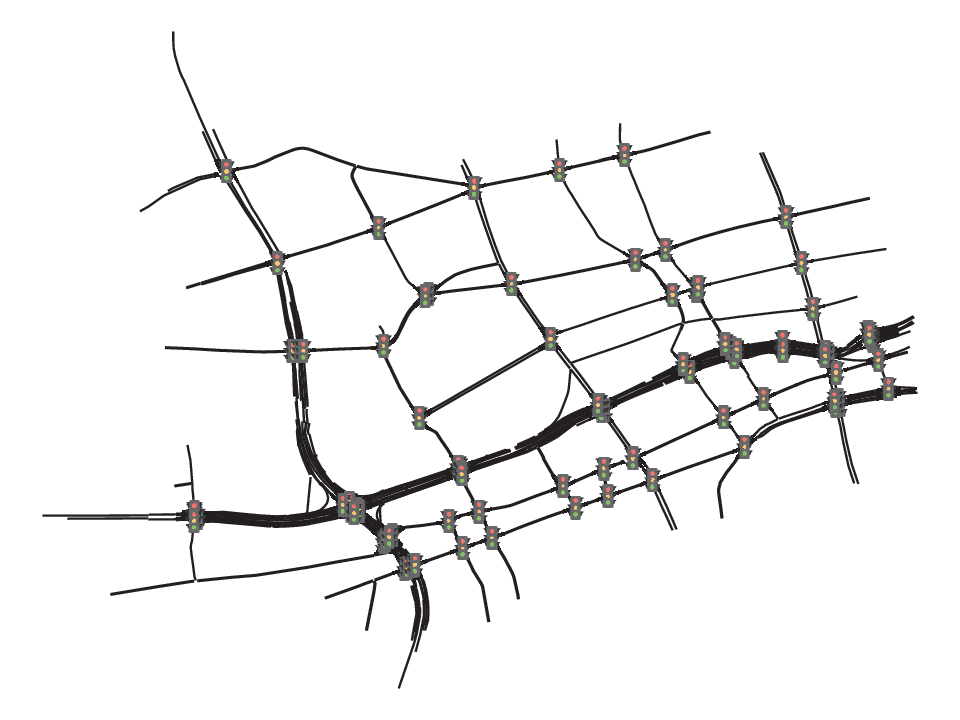}}
      \subfloat[The map of Hangzhou]{\includegraphics[width = 0.33\textwidth]{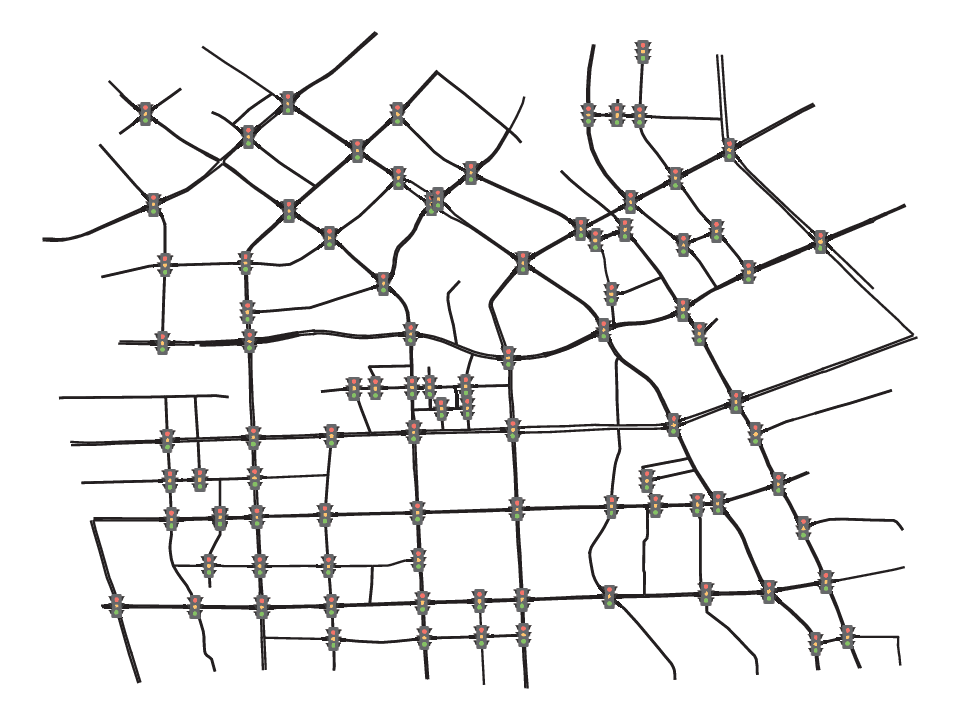}}
    \caption{The map of Shenzhen, Shanghai, and Hangzhou.}
    \label{5-fig-map}
 \end{figure*}
 In this section, we evaluate the performance of TGN-TMoE-enabled large-scale MARL in TSC tasks. As illustrated in Fig. \ref{5-fig-map}, our simulations replicate urban real-world traffic intersections in three representative cases, i.e., ``small-scale map with high traffic density” (Shenzhen),``small-scale map with moderate traffic density” (Shanghai), and ``large-scale map with high traffic density” (Hangzhou).
 Specifically, the Shenzhen and Shanghai maps, covering $5\text{km} \times 5\text{km}$, contain $20$ and $50$ traffic light agents, respectively. 
 The Hangzhou map spans an area of $10\text{km} \times 10\text{km}$ and includes approximately $100$ signalized intersections.
 For simplicity and computational efficiency, all three maps retain only the main drivable roads, omitting minor lanes and non-motorized paths.

 The experiments are conducted using the standard traffic simulation platform SUMO \cite{SUMO2018}, which provides high-fidelity vehicle dynamics and flexible traffic network configurations.
 All algorithms are implemented within the FLOW framework \cite{wu2021flow}, which integrates microscopic traffic simulation with deep RL algorithms, facilitating the construction of complex traffic control scenarios.
 Traffic flow is generated with SUMO’s built-in traffic tool, which assigns each vehicle a unique random route: Shenzhen runs for $1,200$ seconds with $2,300$ vehicles; Shanghai for $600$ seconds with more than $1,200$ vehicles; and Hangzhou for $800$ seconds with $4,000$ vehicles.  This setting ensures diversity in traffic patterns and improves the generalization capability of the learning environment. 
 Notably, these durations are intentionally set to half of the full simulation time to balance data diversity and training efficiency.
 A summary of the experiment configurations is presented in Table \ref{5-tab-eset}. 

 \begin{table}[tbp]
    \centering
    \caption{Environment configurations of each map.}\label{5-tab-eset}
    \begin{tabular}{c|cccc}
    \toprule
    City & Scale & Duration & No. of Veh & No. of TLC \\ \midrule
    Shenzhen & $5\times 5$ & $1,200$ & $> 2,300$ & $20$ \\ \hline
    Shanghai & $5\times 5$ & $600$ & $> 1,200$ & $50$ \\ \hline
    Hangzhou & $10\times 10$ & $800$ & $> 4,000$ & $92$ \\ \bottomrule
    \end{tabular}
\end{table}
 
 All simulations use a fixed time-step of $1$ second, and unless stated otherwise, agents make decisions at each step. 
 During training, each iteration corresponds to a complete simulation episode, and all models are trained for at least $600$ iterations. 
 To accelerate data collection and model updates, we employ $15$ parallel environments. 
 Accordingly, all subsequent results reported in this paper are averaged across these parallel environments.
 
 \subsubsection{Algorithmic Settings}
  \begin{table}[!t]
  \centering
  \caption{Performance comparison under varying numbers of TMoE experts on the Shenzhen map. \textbf{Bold} values are best results, the \textit{italics} indicate second-best results.}
  \begin{tabular}{lcccc}
    \toprule
    No. of Experts & 2       & 4              & 8       & 16            \\
    \midrule
    Fuel           & 14.212  & {14.073} & \textit{14.069}  & \textbf{14.023}  \\
    $\mathrm{CO_2}$ & 334.189 & {330.900} & \textit{330.824} & \textbf{329.736} \\
    Duration       & 195.974 & {196.723} & \textit{193.742} & \textbf{189.863} \\
    Delay          & 89.497  & {88.679}  & \textbf{87.231}  & \textit{87.883}  \\
    \bottomrule
  \end{tabular}
  \label{5-tab-numE}
\end{table}
 For TGN-TMoE, the optimal number of experts therein is set as $16$ due to its superiority on the Shenzhen traffic network in Table~\ref{5-tab-numE}. For decision making, we adopt the MAPPO with the definitions of state, action, and reward as below.
 
 \textbf{State}: The state primarily involves both the traffic light status and the surrounding traffic flow conditions.
 The former regulates traffic conditions in each direction at the controlled intersection and is sequentially packed into a vector, where ``$1$" represents a red light, ``$0$" is a green light, and ``$\frac{1}{2}$" denotes a yellow light.
 In addition to the current signal phase, the state also includes auxiliary information such as the upcoming phase in each direction and the remaining time of the current phase, which are critical for accurate decision-making.
 Regarding traffic flow, the state incorporates essential features of each lane including the queue length, the average speed of vehicles, and the directional layout of incoming and outgoing.
 All components are concatenated and supplemented with necessary alignment masks to construct a unified state vector $\mathbf{s}^{(t)}$, serving as the input to the observation processing and decision-making modules.

 \textbf{Action}: To streamline the action space, we adopt a binary representation that indicates whether the current traffic light phase should be switched to the next phase within a predefined cycle. Specifically, the action value $a=1$ instructs the agent to switch to the next phase, while $a=0$ indicates maintaining the current phase.

 \textbf{Reward}: We design a reward function $\mathcal{R}^{(t)} = -\alpha \mathcal{R}^{(t)}_{\text{sta}} -\beta \mathcal{R}^{(t)}_{\text{wait}}$ with tunable hyperparameters $\alpha > 0$ and $\beta > 0$, that accounts for both traffic flow efficiency $\mathcal{R}^{(t)}_{\text{sta}}$ and the average waiting time of vehicles $\mathcal{R}^{(t)}_{\text{wait}}$. Consistent with ``pressure” \cite{wei2019presslight}, $\mathcal{R}^{(t)}_{\text{sta}}$ evaluates the imbalance in vehicle density between incoming and outgoing lanes with
\begin{equation}
    \mathcal{R}^{(t)}_{\text{sta}} = \big| \sum\limits_{l \in \mathcal{L}_\text{in} }\frac{x(l,t)}{x_{\max}}- \sum\limits_{l \in \mathcal{L}_\text{out} }\frac{x(l,t)}{x_{\max}}\big|,
\end{equation}
 where $\frac{x(l,t)}{x_{\max}}$ represents the normalized vehicle density for a lane $l$ in the in-lane set $\mathcal{L}_\text{in}$ or out-lane set $\mathcal{L}_\text{out}$ at $t$.
 Here, $x(l,t)$ denotes the number of vehicles current queued on lane $l$, reflecting the instantaneous queue length, and $x_{\max}$ is the maximum lane capacity in the map used for normalization. 
 In addition, $\mathcal{R}^{(t)}_{\text{wait}}$ records the cumulative delay on each incoming lane up to time $t$ with $T_w(l, t) = \{T^l_{v_0}, \cdots, T^l_{v_m}\}, \forall l \in \mathcal{L}_\text{in}$ \cite{zhang2024survey}, where $m \in\{1, \cdots, M\}$ denotes the vehicle ID on incoming lane $l$ and $M$ is the total vehicle number in that lane. 
 However, due to limitations in rule-based control, rare events such as collisions may cause unrealistically high delays, which are typically handled by resetting the environment.
 Nevertheless, frequent environment resets may undermine the diversity of samples for MARL training in large-scale simulations.
 To address this, we introduce a waiting threshold $T_{\text{tp}}$ to ``teleport” vehicles that exceed a rational waiting limit and normalize the reward to prevent training instability caused by sparsely scaled rewards. In other words,
\begin{equation}
  \mathcal{R}^{(t)}_{\text{wait}} = \frac{1}{M T_{\text{tp}}} \sum\limits_{l \in \mathcal{L}_{in} }\sum\limits_{m=1}^{M} {T^l_{v_m}}.
\end{equation}
 
 Finally, default parameter settings are summarized in Appendix C.

\subsubsection{Baselines}
We adopt some prevalent and SOTA baselines in all three scenarios.

\textbf{FixedTime:} 
The FixedTime baseline operates on a static timing plan that remains unchanged regardless of varying traffic conditions, representing typical urban traffic systems without advanced traffic signal controllers or RL-based CAV coordination. All signal durations are pre-generated by SUMO and vary across intersections.

\textbf{GLOSA:} 
 GLOSA \cite{suzuki2018new} is a non-DRL method that jointly manages traffic light signals and CAVs. It adjusts the speed of CAVs based on the current traffic light phase and the status of each CAV. In our experiments, GLOSA is integrated with adaptive traffic light controllers for joint control. Phase switching is triggered after detecting sufficient time gaps between successive vehicles, resulting in dynamic phase durations.

\textbf{PressLight:} 
PressLight \cite{wei2019presslight} is a classical DRL-based model proposed to optimize traffic light signals for improving intersection throughput. The model’s state includes the number of vehicles on incoming and outgoing roads. Its reward function leverages the concept of ``pressure,” which is specifically designed to enhance intersection throughput.

\textbf{CoTV:} 
CoTV \cite{10144471} is a cooperative control framework that simultaneously optimizes traffic light signals and CAVs through MAPPO. This model trains TLCs and CAVs separately while maintaining efficient coordination through state information exchange within each intersection’s communication range. Such coordination ensures that vehicles can promptly respond to signal changes, thereby reducing traffic reaction latency and substantially enhancing overall road throughput.

\subsubsection{Evaluation Metrics}
We consider the following $4$ metrics:

\textbf{Fuel Consumption (L/100km):} Fuel consumption is defined as the average amount of fuel (in liters) consumed per $100$ kilometers traveled. Lower values generally correspond to higher vehicle speeds and smoother acceleration patterns. In our experiments, fuel consumption is computed using SUMO’s default vehicle emission model.
  
\textbf{$\ensuremath{\mathrm{CO_2}}$ Emissions (g/km):} $\ensuremath{\mathrm{CO_2}}$ emissions demonstrate the average amount of carbon dioxide emitted, measured in grams per kilometer traveled by all vehicles. 
The corresponding calculations are also based on SUMO's default vehicle emission model.

\textbf{Travel Time (seconds):} Travel time refers to the duration each vehicle takes to complete its trip within the road network. The average travel time is calculated based on vehicles that successfully complete their trips in a given scenario.
  
\textbf{Delay (seconds):} Delay quantifies the difference between a vehicle’s actual travel time and its ideal travel time, which is defined as the duration required to complete a trip at the maximum permitted speed. 

\subsection{Results}

\subsubsection{Performance Comparison}

\begin{table*}[!t]
\centering
\caption{Performance comparison of different methods across three traffic scenarios (i.e., Shanghai, Shenzhen, Hangzhou). 
\textbf{Bold} values indicate the best performance for each metric, while \textit{italic} values imply the second-best.}
\begin{tabular}{l|l|cccccccc}
\toprule
\textbf{City} & \textbf{Metric} & Fixed & GLOSA & PressLight & CoTV & TGN & TTGN & TGN-MoE & TGN-TMoE \\
\midrule
\multirow{4}{*}{Shenzhen}
 & Fuel     & 17.368  & 17.157  & 14.905    & \textit{14.381} & 15.400 & 15.371  & 15.065  & \textbf{14.023} \\
 & $\mathrm{CO_2}$      & 408.382 & 403.425 & 350.490   & \textit{338.162} & 362.112 & 361.442 & 354.226 & \textbf{329.736} \\
 & Travel Time & 233.174 & 229.737 & 208.489   & \textit{198.533} & 213.193  & 211.428 & 208.819 & \textbf{189.863}\\
 & Delay    & 131.596 & 128.248 & 101.895   & \textit{93.793} & 107.398 & 106.665  & 103.329 & \textbf{87.883} \\
\midrule
\multirow{4}{*}{Shanghai}
 & Fuel     & 17.866  & 17.922  & 15.624    & \textbf{14.770} & 15.281 & 15.308  & 15.119 & \textit{14.844} \\
 & $\mathrm{CO_2}$      & 420.090 & 421.410 & 367.389  & \textbf{347.312} & 359.327 & 359.950  & 355.515 & \textit{349.044} \\
 & Travel Time & 258.885 & 259.398 & 244.368   & \textbf{231.485} & 238.211 & 237.412 & 235.598 & \textit{231.556} \\
 & Delay    & 142.193 & 142.956 & 119.236   & \textbf{106.027} & 114.217 & 114.438 & 114.417 & \textit{106.875} \\
\midrule
\multirow{4}{*}{Hangzhou}
 & Fuel     & 12.025  & 12.003  & 12.229    & \textbf{11.979} & 12.219 & 12.211   & 12.195  & \textit{12.003} \\
 & $\mathrm{CO_2}$      & 282.767 & 282.244 & 350.490   & \textbf{281.676} & 287.316 & 287.131 & 286.746 & \textit{282.215} \\
 & Travel Time & 324.725 & 325.043 & 313.843   & {312.734} & 309.821 & \textit{310.101} & 312.402 & \textbf{309.800} \\
 & Delay    & 119.467 & 119.547 & 117.937   & \textit{113.128} & 116.642 & 116.356  & 116.721 & \textbf{112.481} \\
\bottomrule
\end{tabular}
\label{5-tab-merged-city-comparison}
\end{table*}   
\begin{figure*}[tbp]
    \centering
    \subfloat[Shenzhen]{\includegraphics[width = 0.33\textwidth]{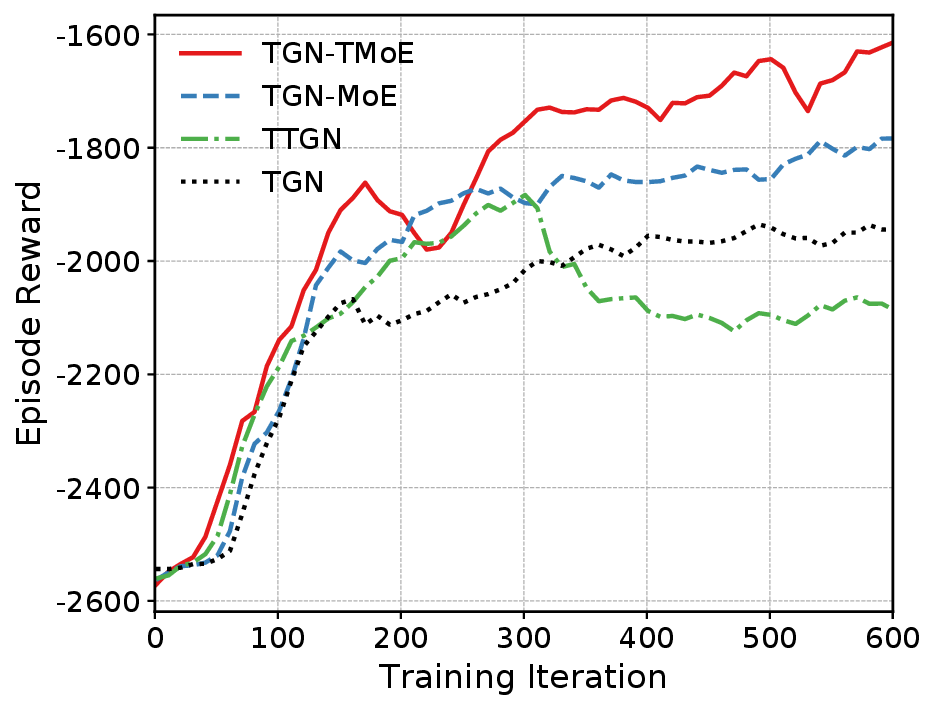}}
    \subfloat[Shanghai]{\includegraphics[width = 0.33\textwidth]{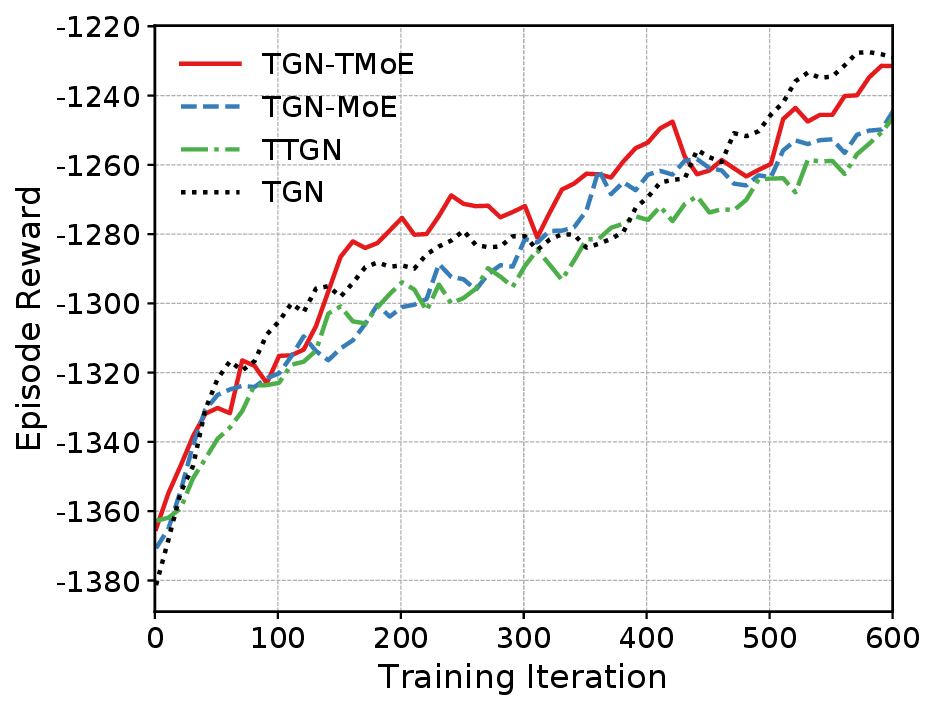}}
    \subfloat[Hangzhou]{\includegraphics[width = 0.33\textwidth]{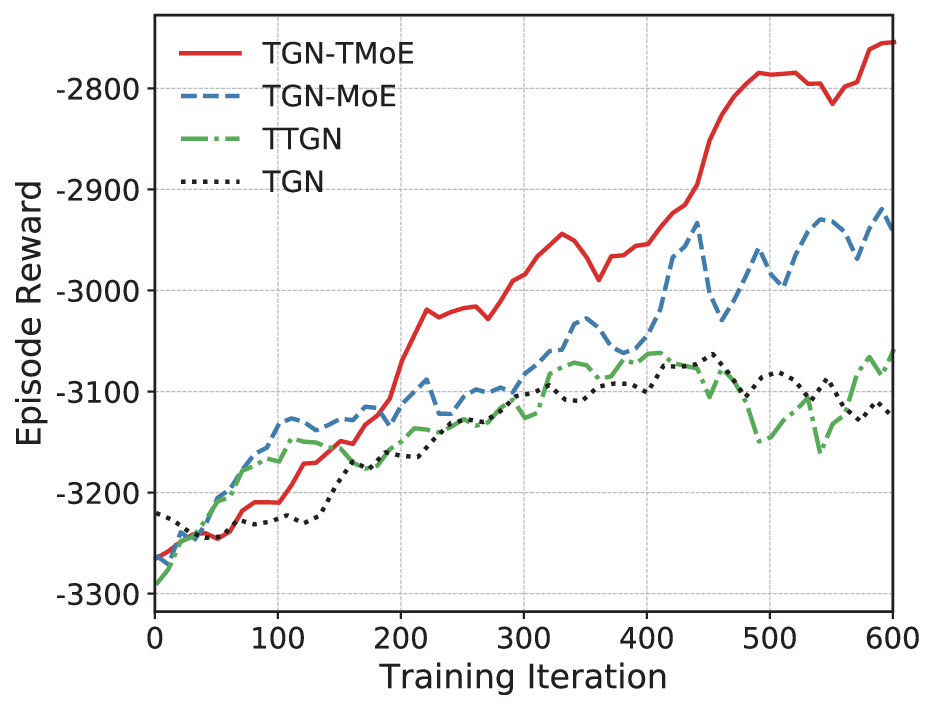}}
    \caption{Episode reward evolution in Shenzhen, Shanghai, and Hangzhou.}
    \label{5-fig-reward}
 \end{figure*}

We summarize the performance comparison between TGN-TMoE and other baselines in Table \ref{5-tab-merged-city-comparison}. It can be observed that TGN-TMoE achieves the best or second-best performance across nearly all metrics in Shenzhen, Shanghai, and Hangzhou. 
Notably, in the Shenzhen scenario, our model outperforms all baselines on every core metric, highlighting its superior adaptability to complex topologies and dynamic traffic conditions.
While TGN, TTGN and TGN-MoE show moderate improvements over the classical TSC methodologies, they still fall short of TGN-TMoE, highlighting the benefit of incorporating topological priors and routing mechanisms in expert allocation.
In Shanghai, CoTV slightly outperforms the TGN-TMoE framework and similar results also apply for 
Hangzhou, the ``large-scale map with high traffic density” scenario. 
The reasons lie in CoTV's dense sensing to quickly respond to traffic changes through vehicle-to-infrastructure communication. 
However, such communication limits its scalability and practicality in larger urban networks.
By contrast, TGN-TMoE offers a topology-aware yet infrastructure-independent solution that yields comparable performance and is more suitable for real-world deployment and generalization across diverse scenarios.
The fact further confirms the necessity of TMoE-based decision and observation refinement in structurally heterogeneous environments. 
\subsubsection{Ablation Studies of TGN-TMoE}

\begin{figure*}[tbp]
    \centering
    \subfloat[Shenzhen]{\includegraphics[width = 0.33\textwidth]{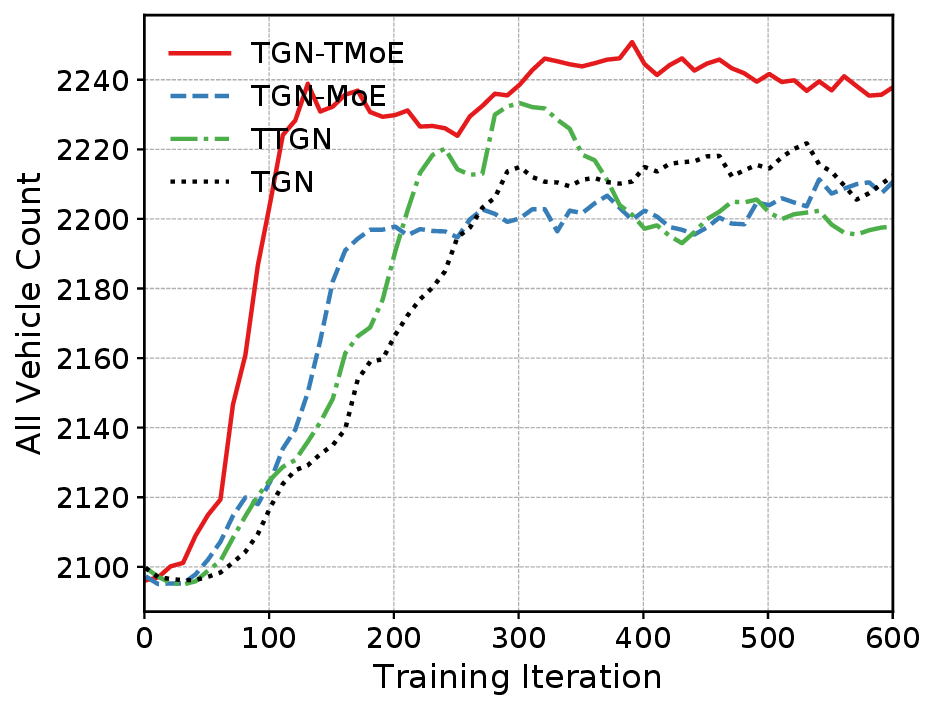}}
    \subfloat[Shanghai]{\includegraphics[width = 0.33\textwidth]{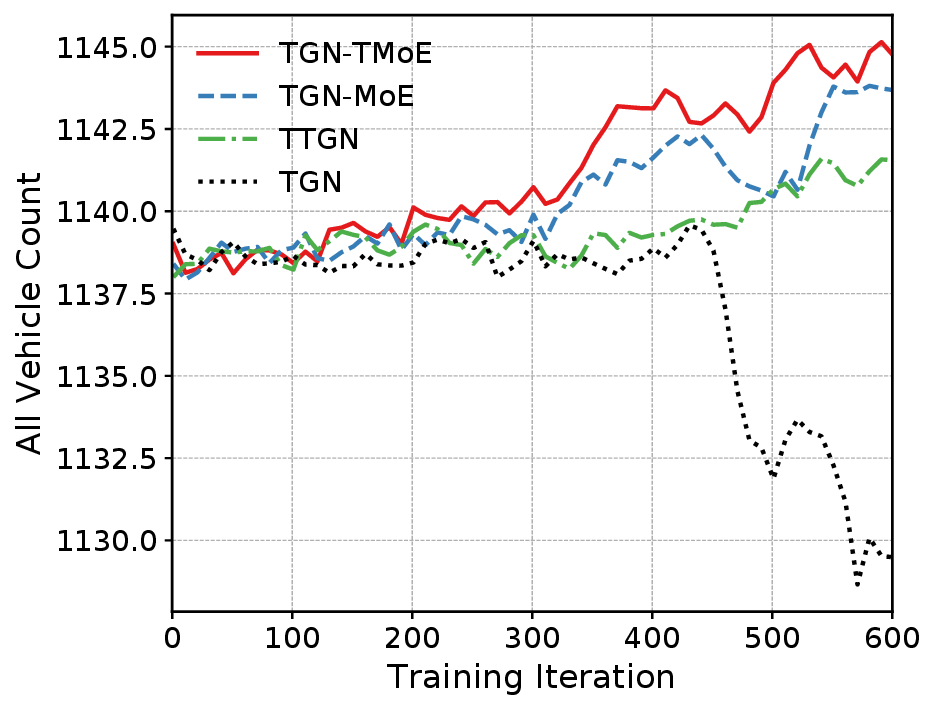}}
    \subfloat[Hangzhou]{\includegraphics[width = 0.33\textwidth]{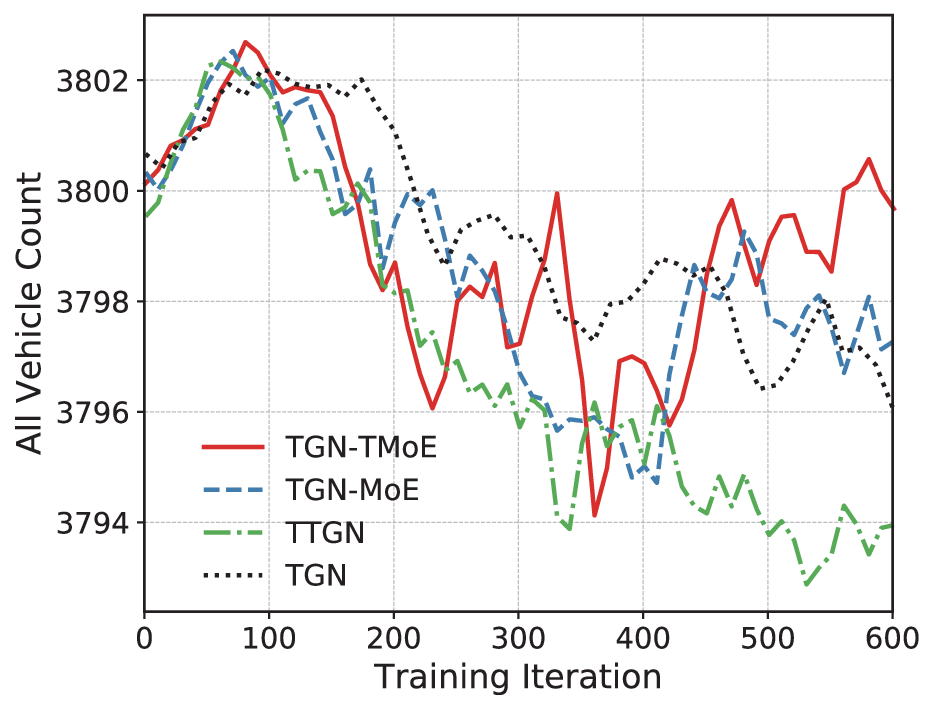}}
    \caption{The involved vehicle number in Shenzhen, Shanghai, and Hangzhou.}
    \label{5-fig-num_v}
 \end{figure*}
 \begin{figure*}[tbp]
  \centering
  \subfloat[Shenzhen]{\includegraphics[width = 0.33\textwidth]{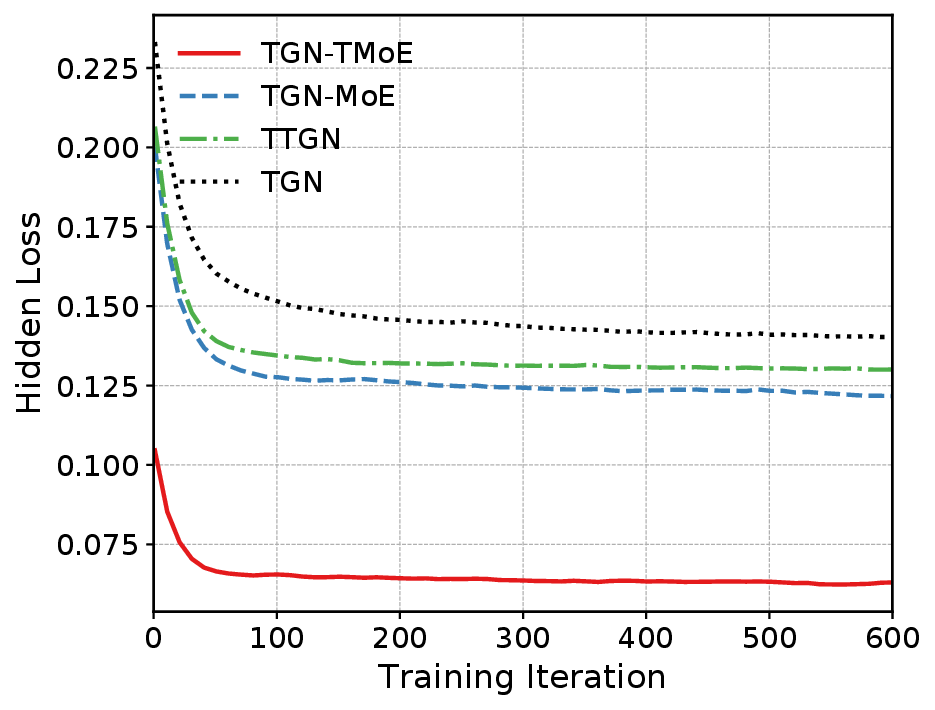}}
  \subfloat[Shanghai]{\includegraphics[width = 0.33\textwidth]{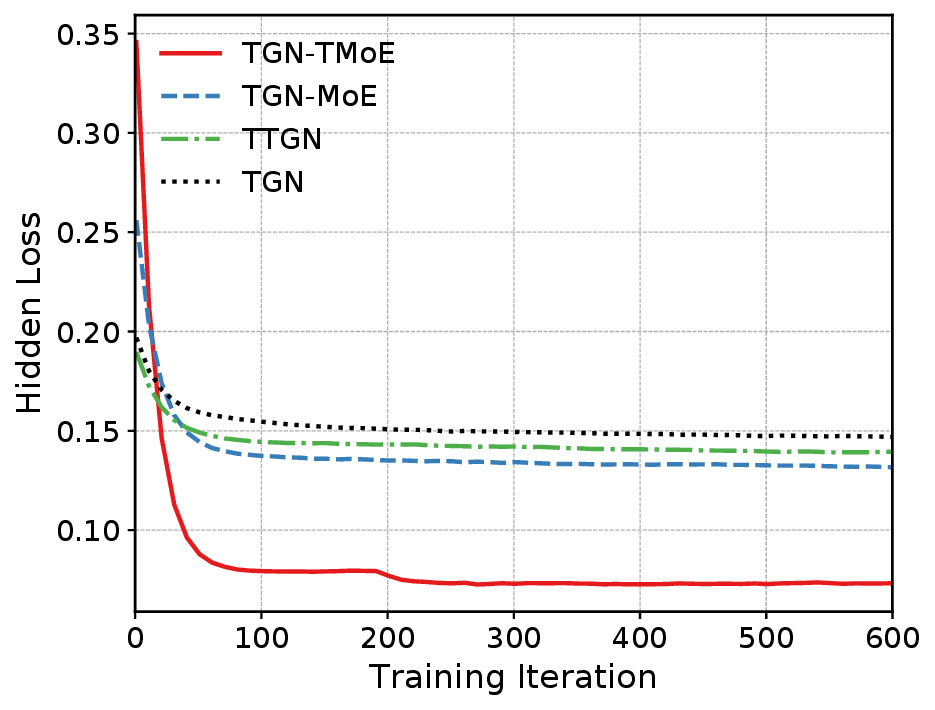}}
  \subfloat[Hangzhou]{\includegraphics[width = 0.33\textwidth]{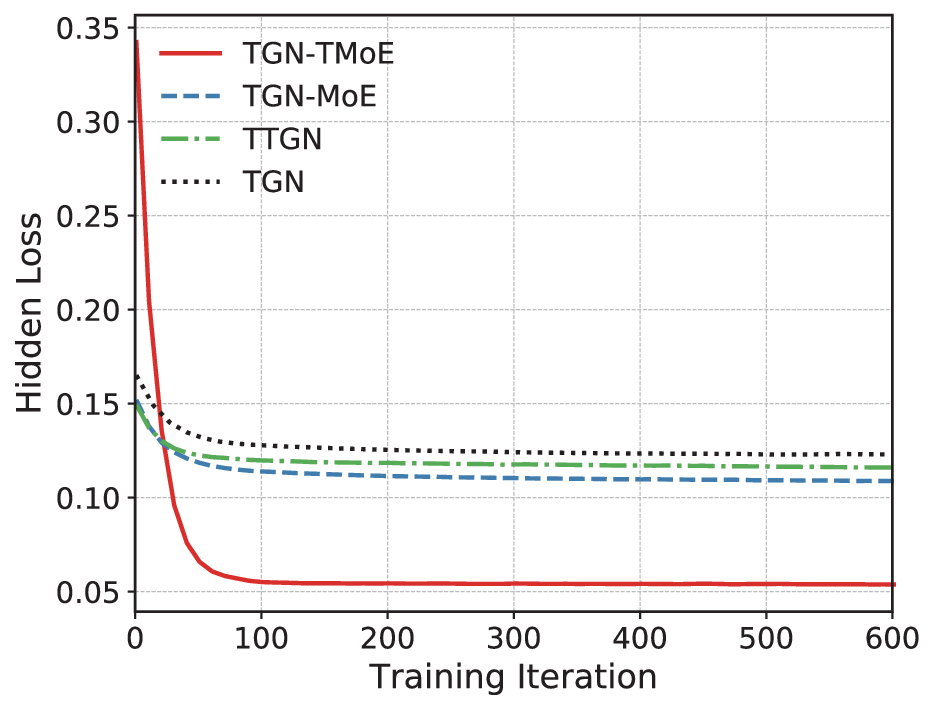}}
  \caption{MSE loss during training.}
  \label{5-fig-loss}
\end{figure*}
To further validate the effectiveness, we conduct a series of ablation experiments by progressively modifying or simplifying the core components of our TGN-TMoE. 
Specifically, we compare TGN-TMoE with the original \textbf{TGN} model without any topological enhancement, a topology-enhanced variant (\textbf{TTGN}) that incorporates only the TDA-based structural features, \textbf{TGN-MoE} that integrates a standard mixture-of-experts module. 

Besides the numerical results in Table \ref{5-tab-merged-city-comparison}, Fig. \ref{5-fig-reward} also depicts the reward trends across three urban traffic scenarios.
As shown in the figure, TGN-TMoE achieves consistently higher average rewards and exhibits faster convergence with greater training stability in Shenzhen. 
However, in the Shanghai and Hangzhou scenarios, the original TGN model unexpectedly yields the highest average rewards during the final training stages.
To interpret such phenomenon, we additionally monitor the number of participating vehicles at each time step. 
The results in Fig. \ref{5-fig-num_v} reveal that, under the TGN model, the number of active vehicles in the environment decreases significantly in the later phases of training. 
Considering the reward function is based on the absolute value of pressure imbalance, this may suggest that MAPPO with TGN implicitly learns a shortcut policy that incurs increased congestion in both incoming and outgoing lanes. 
Such congestion suppresses the pressure metric, thereby inflating the reward without improving traffic flow.
Although TGN appears to outperform other methods in terms of raw reward values in these cases, the underlying traffic dynamics suggest compromised throughput and poor system behavior. 
In contrast, TGN-TMoE maintains stable traffic participation while achieving favorable reward trends, highlighting its robustness and practical effectiveness.
Overall, these findings underscore the importance of structural awareness in traffic signal control. 
The TSD routing mechanism in TGN-TMoE plays a critical role in enhancing the structure learning for TGN, thereby improving the model’s representational capacity and decision-making reliability in complex dynamic environments.

To validate the contribution of individual modules like TDA and MoE, we compare the MSE loss during training, as shown in Fig. \ref{5-fig-loss}. 
The results demonstrate that TGN-TMoE consistently achieves the lowest loss across all scenarios, highlighting its superior capability in graph representation learning.
Furthermore, a clear downward in loss can be observed as more structural components (e.g., TDA and MoE) are incrementally introduced into the TGN. 
The consistent reduction in training loss validates the necessity and effectiveness of each proposed module. 

\begin{figure*}[tbp]
    \centering
    \subfloat[TGN-MoE]{\includegraphics[width = 0.4\textwidth]{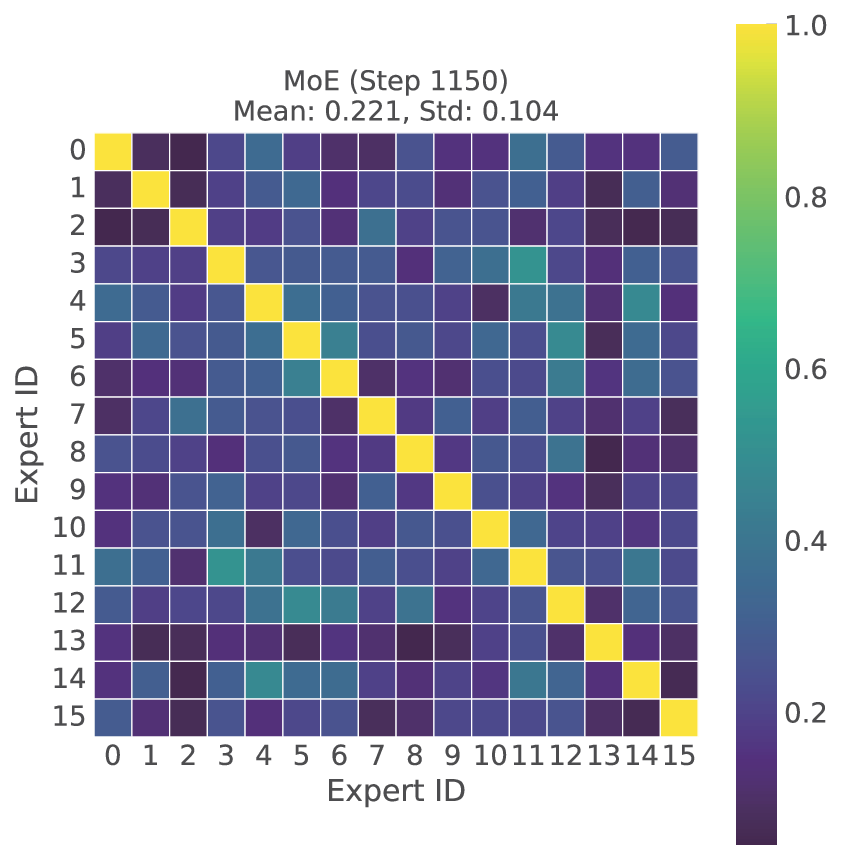}}
    \subfloat[TGN-TMoE]{\includegraphics[width = 0.4\textwidth]{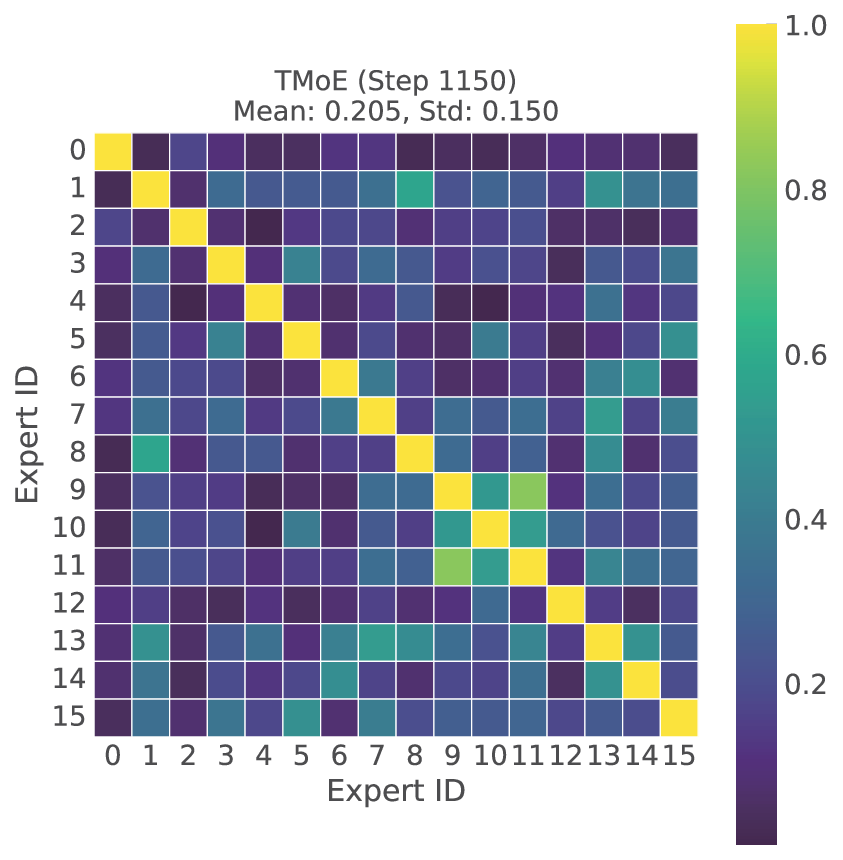}}
    \caption{The comparisons of expert output similarities between TGN-MoE and TGN-TMoE architectures in Shenzhen.}
    \label{5-fig-sim}
\end{figure*}
\subsubsection{Analyses of TGN-TMoE}

To further understand the representational enhancement of TGN-TMoE, Fig. \ref{5-fig-sim} displays the similarity heatmaps of expert outputs under the TGN-MoE and TGN-TMoE architectures, which are computed based on data collected at a stable stage near the end of training.
It can be observed that TGN-TMoE consistently exhibits lower inter-expert similarity than its TGN-MoE counterparts, as reflected by the overall darker tone in the heatmaps and stronger contrast between off-diagonal elements. 
It indicates reduced similarity and greater output diversity among experts, suggesting a higher degree of functional specialization.
Quantitatively, TGN-TMoE yields lower mean similarity scores (e.g., $0.205$ for TGN-TMoE vs. $0.221$ for TGN-MoE) and generally higher standard deviations across scenarios (e.g., $0.150$ for TGN-TMoE vs. $0.104$ for TGN-MoE), supporting the observation that the TSD routing mechanism encourages expert diversity. 

Fig. \ref{5-fig-g} illustrates the local topological structure and expert routing behavior at a selected intersection in Shenzhen. 
Specifically, Fig. \ref{5-fig-g}(a) shows the extracted subgraph centered on the controlled traffic light, where red vertices indicate directly controlled lanes and black vertices represent adjacent ones. 
Based on this subgraph, Fig. \ref{5-fig-g}(b) presents the vertex-level topological similarity matrix, revealing groups of vertices with similar connectivity patterns, such as vertices $0$ and $2$.
Figs. \ref{5-fig-g}(c) and \ref{5-fig-g}(d) display the expert weight distributions under the standard MoE and the proposed TMoE models, respectively. Unlike the imbalanced distribution in standard MoE, vertices with similar topological signatures (e.g., $0$ and $2$) are consistently routed to the same experts with close weights, aligning with the clusters observed in Fig.\ref{5-fig-g}(b).
Moreover, all experts in TMoE are actively engaged, with weights distributed more evenly across the vertex set, avoiding both over-concentration and inactivity.
These results validate that the topology-enhanced routing mechanism in TMoE contributes to robust performance in complex, structure-sensitive traffic signal control tasks.

\begin{figure*}[!t]
    \centering
    \subfloat[Subgraph of the local environment]{\includegraphics[width = 0.4\textwidth]{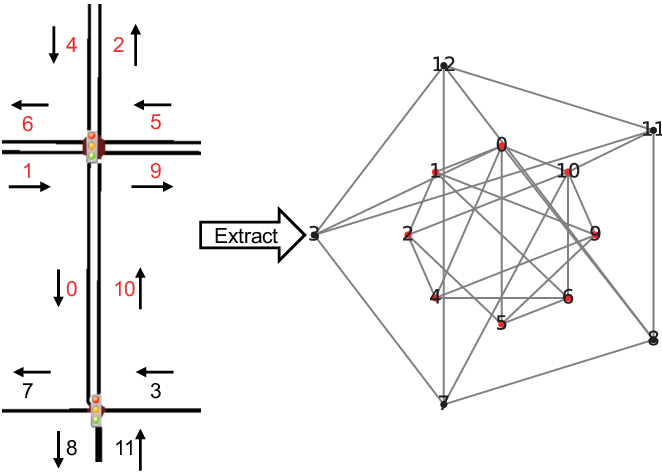}}
    \subfloat[Vertex-level topological similarity]{\includegraphics[width = 0.4\textwidth]{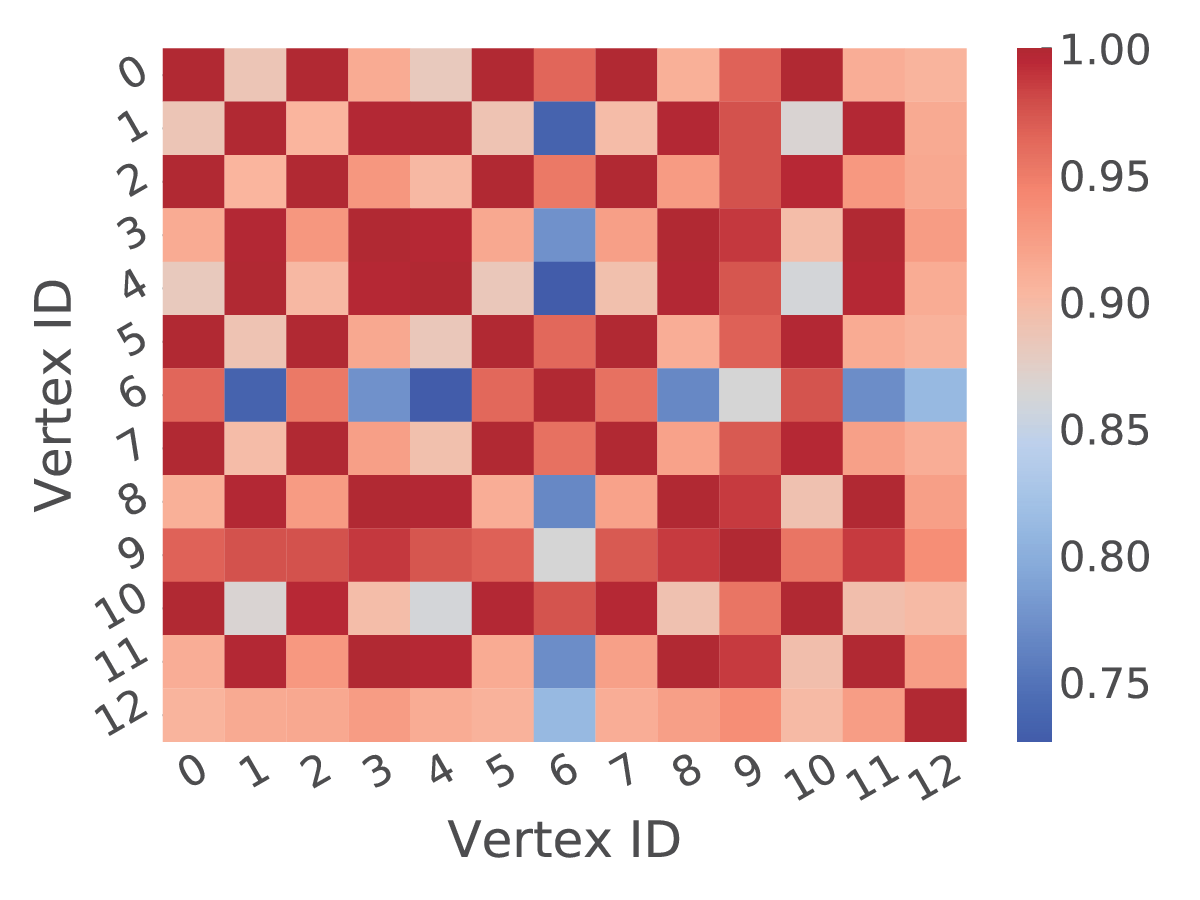}}\\
    \subfloat[Expert weight distribution in MoE]{\includegraphics[width = 0.4\textwidth]{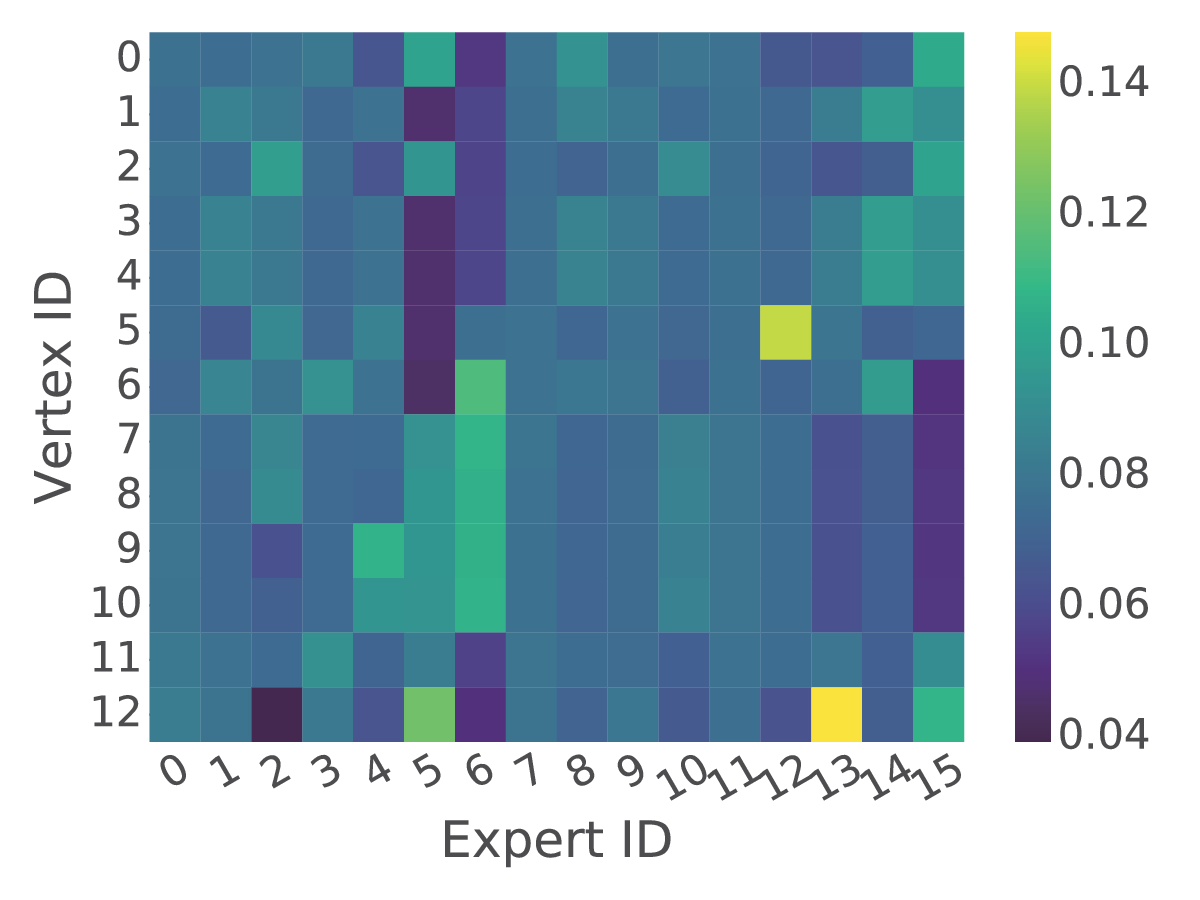}}
    \subfloat[Expert weight distribution in TMoE]{\includegraphics[width = 0.4\textwidth]{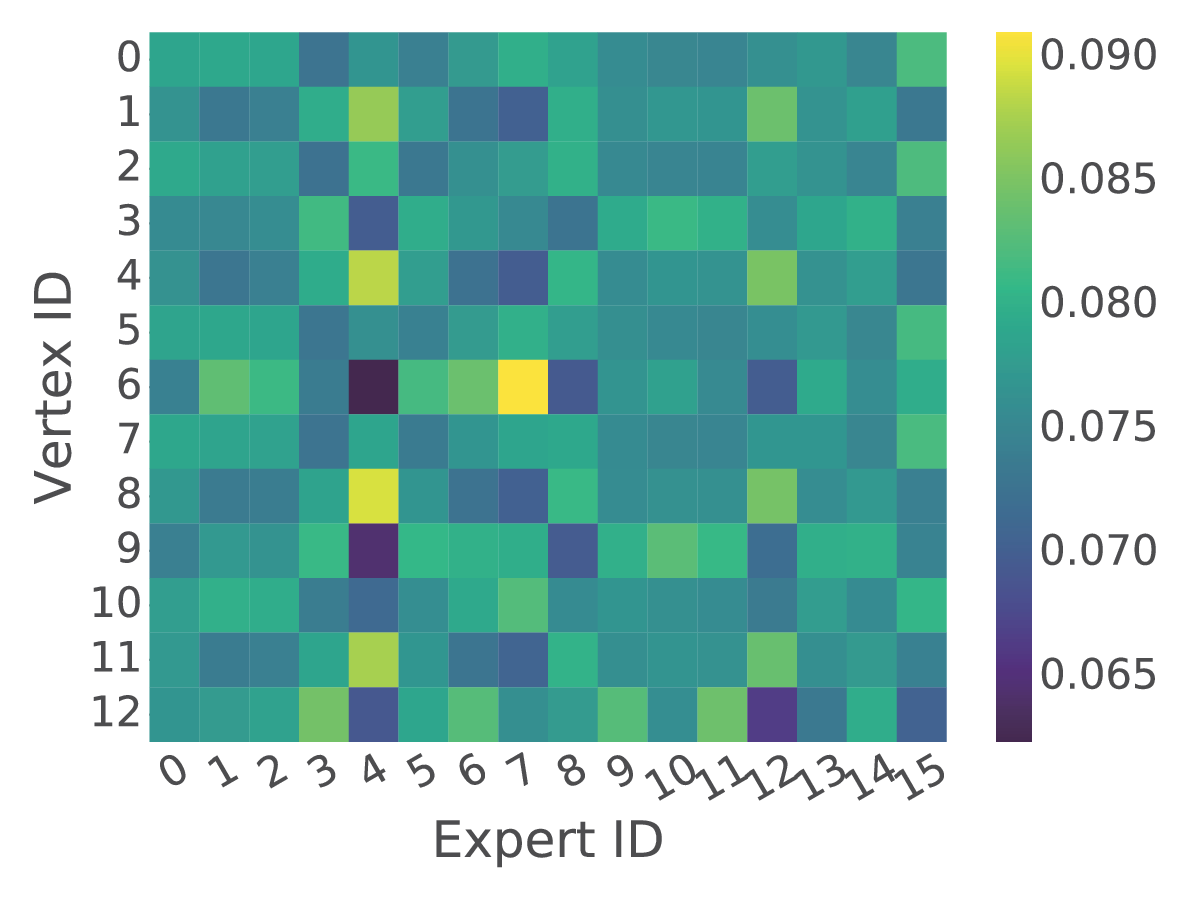}}
    \caption{Illustration of the local topological structure and routing behavior of the model at a randomly selected intersection in Shenzhen.}
    \label{5-fig-g}
\end{figure*}

To validate the expressiveness capabilities of the proposed model, we analyze the hidden vectors fed into the policy module under all three architectures. 
These vectors are projected into a two-dimensional space using t-SNE for visualization.
To ensure diversity, samples are collected from multiple time steps, and their corresponding actions are also recorded for comparison. 
Each point in Fig. \ref{5-fig-tsne} represents an agent with color intensity reflecting the sampling time (lighter indicates earlier).
Red circles highlight instances where $a=1$, corresponding to switching to the next signal phase.
It can be observed that compared to TTGN and TGN-MoE, the incremental involvement of MoE and topology-aware routing in TGN-TMoE yields more distinct separability between $a=1$ samples and others.
Specifically, the $a=1$ samples tend to cluster within a more compact region in the projected space, indicating that the model consistently embeds structurally similar, decision-relevant states into a coherent neighborhood. 
Such tight clustering confirms the stability and generalization ability of TMoE in improving representation quality and structure-aware decision-making.

\begin{figure*}[!t]
    \centering
    
    \subfloat[TTGN]{\includegraphics[width = 0.32\textwidth]{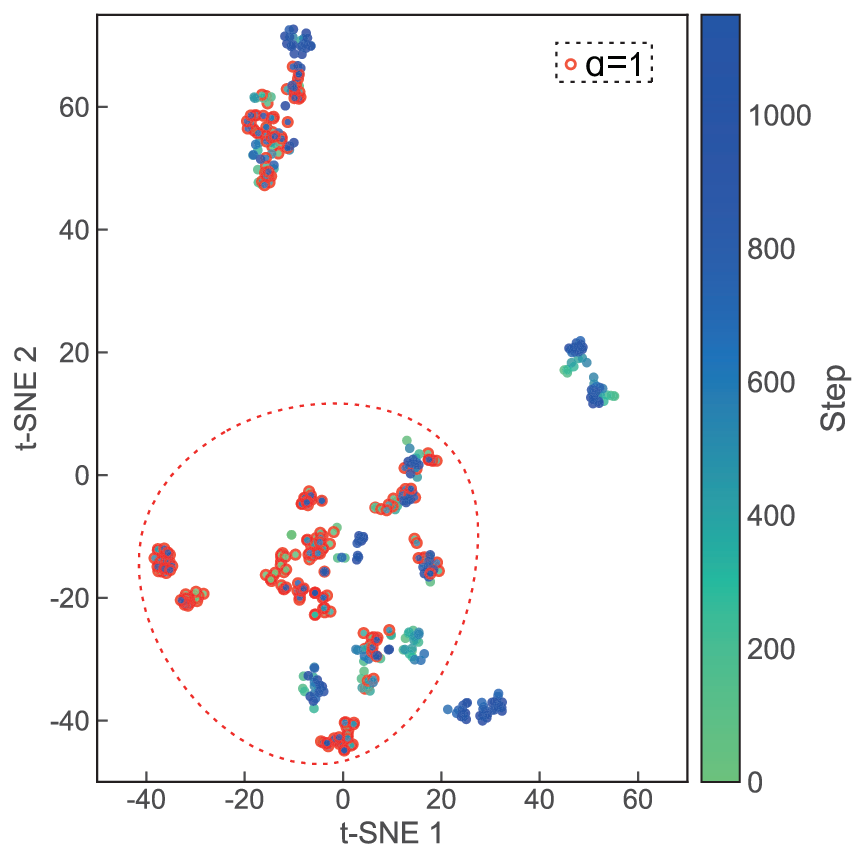}}
    \subfloat[TGN-MoE]{\includegraphics[width = 0.32\textwidth]{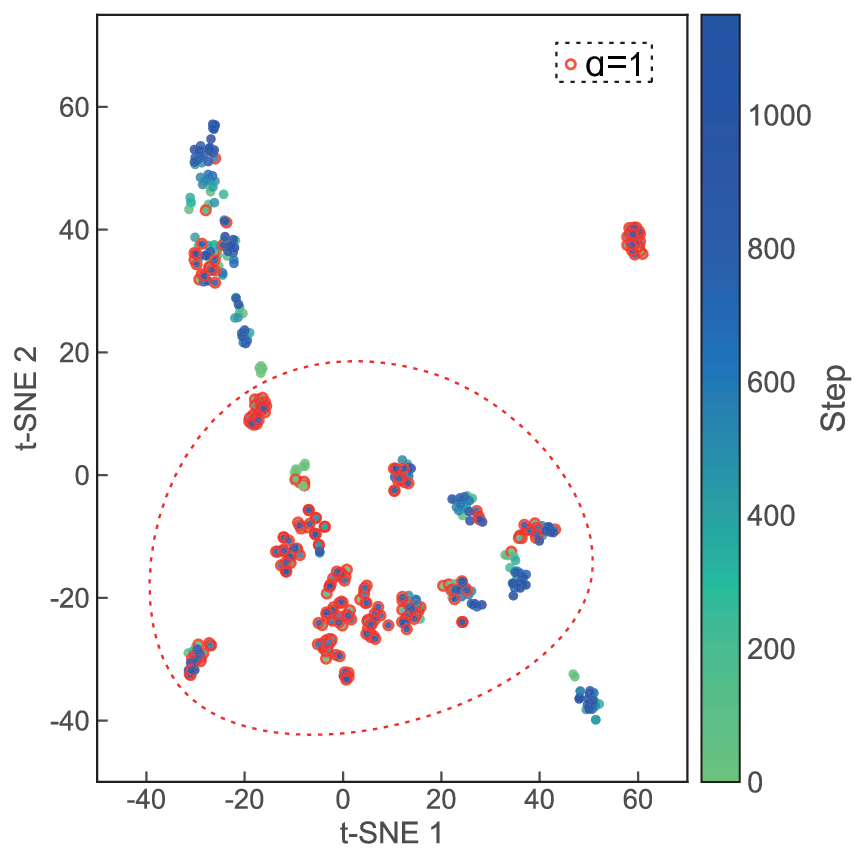}}
    \subfloat[TGN-TMoE]{\includegraphics[width = 0.32\textwidth]{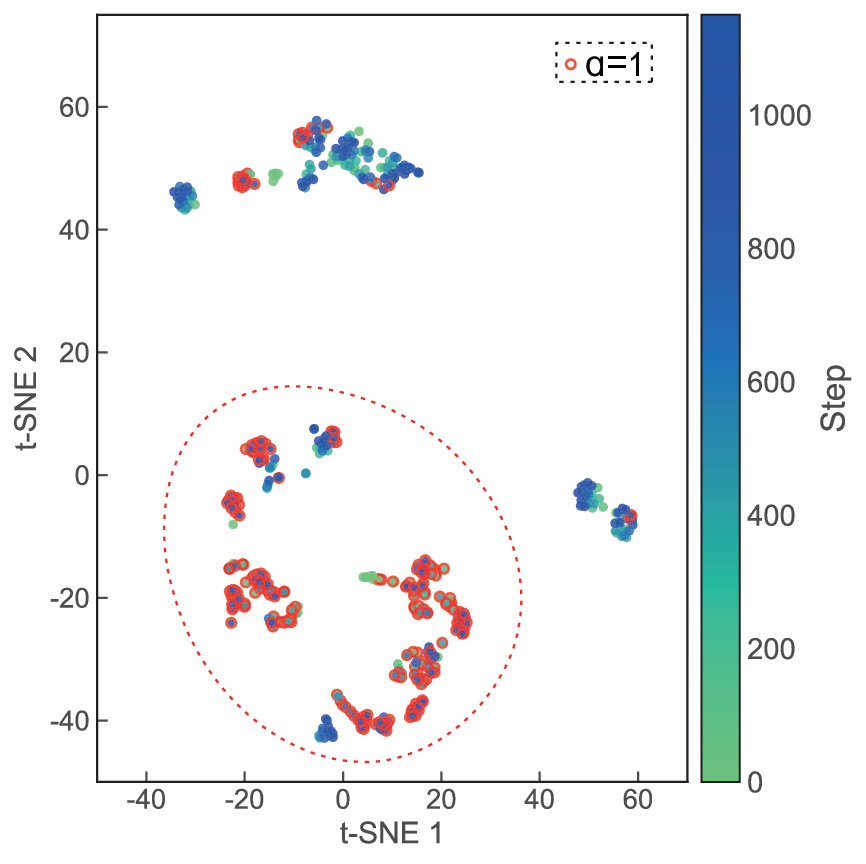}}
    
    \caption{t-SNE visualizations of agent representation distributions over time in Shenzhen using TGN, TGN-MoE, and TGN-TMoE architectures.}
    \label{5-fig-tsne}
\end{figure*}

\begin{table}[tbp]
  \centering
  \caption{Performance of {TGN-TMoE} under different control intervals (i.e., $1$s, $3$s, $5$s, $10$s.) in Shenzhen. 
  \textbf{Bold} values indicate the best performance for each metric, showing the impact of control frequency on system efficiency.}
  \begin{tabular}{l|cccc}
  \toprule
  \textbf{Metric} & \textbf{1s} & \textbf{3s} & \textbf{5s} & \textbf{10s} \\
  \midrule
  Fuel     &  \textbf{14.023} & 14.871  & 14.534  & 14.819  \\
  CO2      &  \textbf{329.736} & 349.676 & 341.741 & 348.458 \\
  Duration &  \textbf{189.863} & 206.400 & 204.930 & 206.165 \\
  Delay    &  \textbf{87.883}  & 100.982 & 97.325  & 101.035 \\
  \bottomrule
  \end{tabular}
  \label{5-tab-control-interval-sensitivity}
  \end{table}

To further assess the impact of control frequency on system performance, we evaluate TGN-TMoE under different control intervals ($1$s, $3$s, $5$s, and $10$s) in the map of Shenzhen. 
As shown in Table \ref{5-tab-control-interval-sensitivity}, the $1$s interval achieves the best performance across all key metrics, demonstrating that high-frequency control enables timely responses to traffic dynamics and enhances overall efficiency. 
On the other hand, while high-frequency control improves performance, it also increases computational and operational overhead. 
These findings suggest future directions in developing frequency-aware and adaptive scheduling mechanisms for efficient and sustainable multi-agent TSC.

\section{Conclusions and Future Works}\label{sec6}

In this paper, we have proposed a topology-aware dynamic graph model, TGN-TMoE, which integrates a topology-enhanced mechanism and a mixture-of-experts architecture into TGN to enhance the modeling capacity of DGNNs in TSC tasks.
Specifically, we model urban road networks as dynamic graphs, where vertices represent roads, edges denote road connections, and changing traffic flow states are treated as dynamic vertex features. 
Given the scale and heterogeneity of TSC graphs, we propose to incorporate TDA into the TGN framework and introduce a topology-sensitive TSD routing mechanism to construct a TMoE. 
The TSD mechanism assigns vertices to experts based on subgraph topological features, enabling each expert to specialize in structurally similar rather than spatially adjacent vertices.
We have also theoretically demonstrated that TGN-TMoE achieves at least a capacity comparable to TOGL and outperforms it in certain structure-sensitive tasks. 
Extensive simulations in three real-world scenarios validate the model’s superiority and effectiveness in large-scale TSC tasks, in terms of systematic performance, feature representation, expert specialization, and policy interpretability.

While TGN-TMoE achieves promising outcomes, our findings also indicate that further performance gains may depend on vehicle–infrastructure coordination.
Undoubtedly, the highly dynamic topology between vehicles and infrastructure poses additional challenges for TGN-TMoE.
However, the proposed model’s enhanced capacity for dynamic graph learning is expected to contribute meaningfully to future advances in this domain.

\appendices
\section{The Embedding Functions for Persistence Diagrams}\label{APP:details_pd}
In this paper, aligned with the methodology in Ref. \cite{horn2021topological}, we adopt some differentiable embedding functions for computing $\Psi$ from an interval $(b,d)$, including \textit{rational hat} function \cite{JMLR:v20:18-358}, \textit{triangle point} transformation, \textit{Gaussian point transformation} and \textit{line point transformation} \cite{carriere2020perslay}.

\begin{definition}[Triangle Point Transformation]
  The \textit{Triangle Point Transformation} maps a point $p = (b, d)$ from a persistence diagram to a fixed-dimensional vector using a triangular function. The transformation is defined as
  \begin{equation}
    \Psi_{\Lambda}(p) = \big[\Lambda_p(t_1), \Lambda_p(t_2), \dots, \Lambda_p(t_q)\big]^T,
  \end{equation}
  where $t_1, t_2, \dots, t_q \in \mathbb{R}$ are fixed sampling points, and the triangular function $\Lambda_p(t)$ is given by 
  \begin{equation}
    \Lambda_p(t) = \max\{0, d - |t - b|\}.
  \end{equation}
  \end{definition}
\noindent Typically, this triangle point transformation function emphasizes the local response of point $p$ to the predefined sampling points $t_i$. 
  By adjusting the locations of sampling points, the transformation can focus on different regions of the persistence diagram, allowing for a tailored representation of local topology. 
  
  \begin{definition}[Gaussian Point Transformation]
  The \textit{Gaussian Point Transformation} uses a Gaussian kernel to map a point $p=(b, d)$ from a PD to a fixed-dimensional vector. The transformation is defined as
  \begin{equation}
    \Psi_{\Gamma}(p) = \big[\Gamma_p(t_1), \Gamma_p(t_2), \dots, \Gamma_p(t_q)\big]^T,
  \end{equation}
  where $t_1, t_2, \dots, t_q \in \mathbb{R}^2$ are fixed sampling points, and the Gaussian kernel $\Gamma_p(t)$ is given by
  \begin{equation}
    \Gamma_{b,d}(t) = \exp\left(-\frac{\|p - t\|_2^2}{2\sigma^2}\right).
  \end{equation}
  Here, $\|p - t\|_2$ is the Euclidean distance between $p$ and $t$, and $\sigma > 0$ is the bandwidth parameter that controls the smoothness of the kernel.
  \end{definition}
  \noindent The Gaussian kernel provides a smooth similarity measure between a point $p$ and the predefined sampling points $t_i$. 
  The parameter $\sigma$ controls the bandwidth of the kernel, influencing the degree of locality in the embedding. 
  A smaller $\sigma$ results in sharper responses, emphasizing local topological features near the sampling points, while a larger $\sigma$ smooths the representation, capturing broader and more global characteristics of the persistence diagram. 
  The sampling points $t_i$ can also be adjusted to focus on specific regions of interest, providing a highly flexible and adaptive embedding method suitable for noisy or complex data.

  \begin{definition}[Line Point Transformation]
  The \textit{Line Point Transformation} maps a point $p = (b, d)$ from a PD to a fixed-dimensional vector by projecting it onto a set of predefined lines. The transformation is defined as
  \begin{equation}
    \Psi_{L}(p) = \big[L_{\Delta_1}(p), L_{\Delta_2}(p), \dots, L_{\Delta_q}(p)\big]^T,
  \end{equation}
  where $\Delta_1, \Delta_2, \dots, \Delta_q$ are predefined lines in $\mathbb{R}^2$, and the projection onto a line $\Delta$ is given by
  \begin{equation}
    L_{\Delta}(p) = \langle p, e_{\Delta} \rangle + b_{\Delta}.
  \end{equation}
  Here, $e_{\Delta} \in \mathbb{R}^2$ is the direction vector of the line $\Delta$, and $b_{\Delta} \in \mathbb{R}$ is its bias term.
  \end{definition}
  \noindent This linear point transformation captures the directional characteristics of points in a PD by projecting them onto lines defined by direction vectors $e_{\Delta}$ and offsets $b_{\Delta}$. 
  The choice of $e_{\Delta}$ and $b_{\Delta}$ plays a crucial role in determining which directional features are emphasized.

  \begin{definition}[Rational Hat Transformation]
    The \textit{Rational Hat Structure Transformation} maps a point $p = (b, d)$ from a PD into a vector space while preserving local topology features and enhancing the distinction of short-lived features. The function is defined as
    \begin{equation}
      \Psi_\text{rat}(p;\mu, r, q) = \frac{1}{1 + \|p - \mu\|_q} - \frac{1}{1 + |r - \|p - \mu\|_q|},
    \end{equation}
    where $\mu = (\mu_0, \mu_1)$ defines the center point of focus, $r$ controls the effective influence radius, and $q$ specifies the type of norm used to compute distances (e.g., $q = 2$ for Euclidean distance or $q = 1$ for Manhattan distance). 
  \end{definition}
  \noindent The rational hat transformation function is constructed to emphasize local features by assigning higher weights to points near $\mu$ through the first term $\frac{1}{1 + \|x - \mu\|_q}$, while the second term $\frac{1}{1 + |r - \|x - \mu\|_q|}$ introduces a negative feedback mechanism that attenuates contributions from points farther away or those near the radius $r$. 

  \section{Proof of Theorem \ref{thm:superExp}}\label{APP:superExp}
Essentially, for any graph $\mathcal{G}$, our improved model can be expressed as a composite function 
\begin{equation}
  F(\widehat{\mathbf{V}},\mathbf{T}_v) = f_3(\mathbf{T}_v)f_2\big(f_1(\mathbf{T}_v)^\intercal\cdot \widehat{\mathbf{V}}\big),
  \label{eq:F_improve}
\end{equation}
where $\widehat{\mathbf{V}} \in \mathbb{R}^{H \times N\times d}$ denotes the multi-head output from the GAT-based module (i.e., the TGN model) and $\mathbf{T}_v  \in \mathbb{R}^{N\times d_1}$ is the topological feature (PD representation) obtained via Eq. \eqref{eq:tda}. 
Here, $d$ denotes the dimension of vertex features, $N$ is the number of vertices in $\mathcal{G}$, and $H$ represents the number of GAT heads.
$f_1(\cdot)$ is the topology-based weight mapping function that produces routing scores from $\mathbf{T}_v$, corresponding to the $\text{Route}(\cdot)$ module in TMoE.
$f_2(\cdot)$ represents the expert transformation, aligned with $\text{Proj}(\cdot)$ in TMoE, which operates on the weighted combination of $\widehat{\mathbf{V}}$ and $\mathbf{T}_v$.
$f_3(\cdot)$ denotes the topology-guided fusion function, as formulated in Eq. \eqref{eq:moe_tt_mixing_2_3}, and together with $f_2$ constitutes the implementation of $\text{Merge}(\cdot)$ in our proposed TMoE architecture.

To establish a formal comparison, we consider the representation learning in TOGL, where the TOGL embedding can be abstractly represented as
\begin{equation}
    \widetilde{\mathbf{V}} = f(\mathbf{V},\mathbf{T}_v, \mathcal{G}),
    \label{eq:f}
\end{equation}
where $f$ denotes the topology-aware graph representation function, like GAT.
Our target is to demonstrate that the proposed model achieves superior representational capacity relative to TOGL.
Beforehand, we introduce the following lemmas.

\begin{lemma}[Injectivity of $f_1$]
  Consider the normalized topological feature vector $\mathbf{T}_v \in \mathbb{R}^{N \times d_1}$ satisfying $|\mathbf{T}_v| = 1$. Given a randomly initialized weight matrix $\mathbf{W}_1 \in \mathbb{R}^{d_1 \times H \times P}$ viewed as a linear transformation, and since $f_1$ involves computing softmax along the first dimension of the product, we reshape it into matrices $\widehat{\mathbf{W}}_1^{d_1 \times H \times P}$ and $\widetilde{\mathbf{W}}_1^{d_1 \times (HP)}$. If $\operatorname{rank}(\widetilde{\mathbf{W}}_1) = d_1$ and $\operatorname{rank}[\widetilde{\mathbf{W}}_1, \mathbf{1}_{H \times P}] = d_1+1$, then $f_1(\mathbf{T}_v^1) \neq f_1(\mathbf{T}_v^2)$ for any distinct vectors $\mathbf{T}_v^1 \neq \mathbf{T}_v^2$.
\end{lemma}

\begin{proof}
  We proceed by contradiction. 
  Suppose there exist distinct vectors $\mathbf{T}_v^1 \neq \mathbf{T}_v^2$ such that $f_1(\mathbf{T}_v^1) = f_1(\mathbf{T}_v^2)$, we have Eq. \eqref{eq:lemma_eq_0}.
  \begin{figure*}[pbh]
    \centering
    \noindent\hrulefill\\[0.5ex]
    \begin{equation}
      \begin{aligned}
      \label{eq:lemma_eq_0}
      &\frac{1}{\sum_{k = 1}^H\exp{(\mathbf{T}_v^1\widehat{\mathbf{W}}_{1j})}}[\exp{(\mathbf{T}_v^1\widehat{\mathbf{W}}_{10})},\cdots,\exp{(\mathbf{T}_v^1\widehat{\mathbf{W}}_{1H})}]^\intercal= \frac{1}{\sum_{k = 1}^H\exp{(\mathbf{T}_v^2\widehat{\mathbf{W}}_{1j})}}[\exp{(\mathbf{T}_v^2\widehat{\mathbf{W}}_{10})},\cdots,\exp{(\mathbf{T}_v^2\widehat{\mathbf{W}}_{1H})}]^\intercal, \\
      &\implies \frac{\sum_{k = 1}^H\exp{(\mathbf{T}_v^2\widehat{\mathbf{W}}_{1j})}}{\sum_{j = 1}^H\exp{(\mathbf{T}_v^1\widehat{\mathbf{W}}_{1j})}}[\exp{(\mathbf{T}_v^1\widehat{\mathbf{W}}_{10})},\cdots,\exp{(\mathbf{T}_v^1\widehat{\mathbf{W}}_{1H})}]^\intercal = [\exp{(\mathbf{T}_v^2\widehat{\mathbf{W}}_{10})},\cdots,\exp{(\mathbf{T}_v^2\widehat{\mathbf{W}}_{1H})}]^\intercal. 
      \end{aligned}
      \end{equation}
  \end{figure*}
  
  Taking the natural logarithm on both sides, \eqref{eq:lemma_eq_0} can be re-written as
  \begin{equation}
    \begin{aligned}
    \label{eq:lemma_eq_1}
  &\ln\left[\frac{\sum_{k = 1}^H\exp{(\mathbf{T}_v^2\widehat{\mathbf{W}}_{1k})}}{\sum_{k = 1}^H\exp{(\mathbf{T}_v^1\widehat{\mathbf{W}}_{1k})}}\right]+[\mathbf{T}_v^1\widehat{\mathbf{W}}_{10},\cdots,\mathbf{T}_v^1\widehat{\mathbf{W}}_{1H}]^\intercal \\
  &= [\mathbf{T}_v^2\widehat{\mathbf{W}}_{10},\cdots,\mathbf{T}_v^2\widehat{\mathbf{W}}_{1H}]^\intercal.
    \end{aligned}
  \end{equation}

  Since $\mathbf{T}_v^1 \neq \mathbf{T}_v^2$, we can define $c = \ln\left[\frac{\sum_{k = 1}^h\exp{(\mathbf{T}_v^2\widehat{\mathbf{W}}_{1k})}}{\sum_{k = 1}^h\exp{(\mathbf{T}_v^1\widehat{\mathbf{W}}_{1k})}}\right]$ and obtain $\mathbf{T}_v^1\widehat{\mathbf{W}}_{1k}+c = \mathbf{T}_v^2\widehat{\mathbf{W}}_{1k}, \forall k$, which implies
  \begin{align}
  &\frac{1}{c}(\mathbf{T}_v^2-\mathbf{T}_v^1)\widehat{\mathbf{W}}_{1k} = 1 \implies 
  \frac{1}{c}(\mathbf{T}_v^2-\mathbf{T}_v^1)\widehat{\mathbf{W}}_{1} = \mathbf{1}_{N \times H \times P} \nonumber \\
  &\implies \frac{1}{c}(\mathbf{T}_v^2-\mathbf{T}_v^1)\widetilde{\mathbf{W}}_{1} = \mathbf{1}_{N \times (HP)}.    \label{eq:lemma_eq_2}
  \end{align}

  However, since $\operatorname{rank}[\widetilde{\mathbf{W}}_1, \mathbf{1}_{H \times P}] = d_1+1$, the all-ones vector is not in the column space of $\widetilde{\mathbf{W}}_1$. 
  Therefore, there cannot exist a scalar $\frac{1}{c}$ that makes Eq. \eqref{eq:lemma_eq_2} hold. The contradiction with our initial assumption ensures the injectivity of $f_1$.
\end{proof}
It is worth noting that due to the random initialization of parameters and our implementation where $d_1 \leq HP$, the assumptions, $\operatorname{rank}(\widetilde{\mathbf{W}}_1) = d_1$ and $\operatorname{rank}[\widetilde{\mathbf{W}}_1, \mathbf{1}_{H \times P}] = d_1+1$, can be guaranteed in our research \cite{han2024bridging}. 
Similarly, we can guarantee the injectivity of $f_3$ by the following lemma.
\begin{lemma}[Injectivity of $f_3$]
  Consider the normalized topological feature vector $\mathbf{T}_v \in \mathbb{R}^{N \times d_1}$ satisfying $|\mathbf{T}_v| = 1$. Given a randomly initialized weight matrix $\mathbf{W}_3 \in \mathbb{R}^{d_1 \times H \times P}$ viewed as a linear transformation, and since $f_1$ involves computing softmax along the second dimension of the product, we reshape it into matrices $\widehat{\mathbf{W}}_3^{d_1 \times H \times P}$ and $\widetilde{\mathbf{W}}_1^{d_1 \times (HP)}$. If $\operatorname{rank}(\widetilde{\mathbf{W}}_3) = d_1$ and $\operatorname{rank}[\widetilde{\mathbf{W}}_3, \mathbf{1}_{H \times P}] = d_1+1$, then $f_3(\mathbf{T}_v^1) \neq f_1(\mathbf{T}_v^2)$ for any distinct vectors $\mathbf{T}_v^1 \neq \mathbf{T}_v^2$.
\end{lemma}

Using these Lemmas, we can ensure the injectivity of the composite function $F(\widehat{\mathbf{V}},\mathbf{T}_v)$.
Next, we are ready to prove the Theorem.
\begin{proof}
TOGL possesses expressive power consistent with or even exceeding that of 1-WL \cite{horn2021topological}, on the basis of an injective $f$ in Eq. \eqref{eq:f}. As demonstrated by the preceding lemmas, our modified module in Eq. \eqref{eq:F_improve} still preserves the property of injectivity.

Let $\mathcal{G}_1 = \langle\mathcal{V}_1, \mathcal{E}_1\rangle$ and $\mathcal{G}_2 = \langle\mathcal{V}_2, \mathcal{E}_2\rangle$ be two graphs satisfying the conditions of the theorem. 
TOGL represents these graphs using static embeddings that combine vertex features and global topological features. 
If $\mathcal{G}_1$ and $\mathcal{G}_2$ have identical global topology, like the same global Betti number, while differing in multiple local features, TOGL cannot distinguish between them.

In our improved model, the MH-GAT generates diverse embeddings $\hat{\mathbf{v}}^{(h)}_j$ for each vertex in the $h$-th head. 
Due to the embedding functions of TDA, the generated vertex-level topological representations contain rich local topological information, such as their contributions to cycles or connectivity patterns.
Based on the vertex-level topological signatures, the SoftMoE module then dynamically dispatches graph-level tokens to experts.
For the generation of slots, we can regard it as selecting a specific topological pattern from the vertex-level topological signatures, which is similar to highlighting a series of vertices that own identical topological signatures.
The dynamic assignment allows the model to focus on specific local topological patterns, such as distinct cycles or connectivity substructures. 
If the local topological signatures of two graphs are different, the model will generate different embeddings for vertices within the two graphs, that is,
\begin{align}
    \widetilde{\mathbf{V}}_{j}^{\mathcal{G}_1} \neq \widetilde{\mathbf{V}}_{j}^{\mathcal{G}_2}, \quad \text{if } \mathbf{T}_j^{\mathcal{G}_1} \neq \mathbf{T}_j^{\mathcal{G}_2}.
\end{align}
Thus, the improved model produces embeddings that differentiate $\mathcal{G}_1$ and $\mathcal{G}_2$ based on their local topology, achieving greater expressiveness than TOGL in these scenarios.
\end{proof}

\section{The hyper-parameters list}\label{app:lists}

In our experiments, we adopt the hyperparameter configurations summarized in Table \ref{tab:list}, which are consistent with the settings in Refs. \cite{zhu2024semantics,horn2021topological,pmlr-v235-obando-ceron24b,10144471}. 
These values have been empirically validated to achieve stable training and competitive performance in our large-scale TSC tasks.
\begin{table}[htbp]
\centering
\caption{Hyperparameter Settings for Model Training.}
\label{tab:list}
\begin{tabular}{l l}
\toprule
\multicolumn{2}{c}{\textbf{TGN-TMoE Module}} \\
\midrule
No. of GAT layers & $2$ \\
No. of GAT heads ($H$)& $4$ \\
No. of experts in TMoE ($P$) & $16$ \\
No. of filtrations in TDA ($U$) & $12$ \\
Dimension of raw feature ($d_0$) & $7$ \\
Dimension of TDA feature ($d_1$) & $7$ \\
Dimension of TGN's output ($d$) & $7$ \\
Dimension of readout layer ($d_o$) & $28$ \\
\midrule
\multicolumn{2}{c}{\textbf{MAPPO Module}} \\
\midrule
No. of experts in TMoE ($P$) & $16$ \\
Output dimension of TMoE & $64$ \\
First layer size & $(28, 128)$ \\
Third layer size & $(64, 1)$ \\
Activation function & $\text{Tanh}$ \\
\midrule
\multicolumn{2}{c}{\textbf{Overall Training Settings}} \\
\midrule
No. of parallel envs & $15$\\
GAE lambda ($\lambda$) & $0.97$ \\
Discount factor ($\gamma$) & $0.99$ \\
Clipping parameter ($\epsilon$)& $0.3$ \\
Entropy coeff ($\iota$)& $0.003$ \\
Number of SGD iterations & $5$ \\
Learning rate & $5 \times 10^{-4}$ \\
Shared parameters & \texttt{True} \\
\bottomrule
\end{tabular}
\end{table}
\bibliographystyle{IEEEtran}
\bibliography{reference}
\end{document}